%% file: main_caemra.tex
\documentclass{article}

\input{math_commands.tex}

\usepackage[preprint]{neurips_2023}

\usepackage[utf8]{inputenc} 
\usepackage[T1]{fontenc}    
\usepackage{hyperref}       
\usepackage{url}            
\usepackage{booktabs}       
\usepackage{amsfonts}       
\usepackage{nicefrac}       
\usepackage{microtype}      
\usepackage{xcolor}         

\usepackage{csquotes}

\usepackage{algpseudocode}

\usepackage{url}            
\usepackage{booktabs}       
\usepackage{amsfonts}       
\usepackage{nicefrac}       
\usepackage{microtype}      
\usepackage{amsthm}
\usepackage{chngcntr}
\usepackage{apptools}
\newtheorem{theorem}{Theorem}

\newtheorem{conjecture}[theorem]{Conjecture}

\newtheorem*{remark}{Remark}
\newtheorem*{lemma*}{Lemma}
\newtheorem{Proposition}{Proposition}
\newtheorem*{Proposition*}{Proposition}
\newtheorem{assumption}{Assumption}

\usepackage{graphicx}
\usepackage{graphics}
\usepackage{parskip}
\usepackage{amsmath}
\usepackage[ruled,vlined, linesnumbered, algo2e]{algorithm2e}
\usepackage{algorithm}

\usepackage{subfig}

\usepackage{dsfont}
\usepackage{natbib}
\usepackage[leftcaption]{sidecap}

\usepackage[flushleft]{threeparttable}
\usepackage{array,booktabs,makecell}

\usepackage{amsmath,amssymb}
\usepackage[flushleft]{threeparttable}
\usepackage{array,booktabs,makecell}
\usepackage[font=small]{caption}

\usepackage{wrapfig}
\usepackage[normalem]{ulem}
\usepackage{soul}
\usepackage[normalem]{ulem}               
\newcommand\redout{\bgroup\markoverwith
{\textcolor{red}{\rule[0.5ex]{2pt}{0.8pt}}}\ULon}

\makeatletter
\renewcommand{\Indentp}[1]{%
  \advance\leftskip by #1
  \advance\skiptext by -#1
  \advance\skiprule by #1}%
\renewcommand{\Indp}{\algocf@adjustskipindent\Indentp{\algoskipindent}}
\renewcommand{\Indm}{\algocf@adjustskipindent\Indentp{-\algoskipindent}}
\makeatother

\title{Decentralized Blockchain-based Robust Multi-agent Multi-armed Bandit}

\author{%
  Mengfan Xu$^{1}$ \quad Diego Klabjan$^{1}$ \\
  $^1$Department of Industrial Engineering and Management Sciences, Northwestern University\\
\texttt{MengfanXu2023@u.northwestern.edu, d-klabjan@northwestern.edu}}

\begin{document}

\maketitle

\thispagestyle{empty}

\begin{abstract}
We study a robust, i.e. in presence of malicious participants, multi-agent multi-armed bandit problem where multiple participants are distributed on a fully decentralized blockchain, with the possibility of some being malicious. The rewards of arms are homogeneous among the honest participants, following time-invariant stochastic distributions, which are revealed to the participants only when certain conditions are met to ensure that the coordination mechanism is secure enough. The coordination mechanism's objective is to efficiently ensure the cumulative rewards gained by the honest participants are maximized. To this end and to the best of our knowledge, we are the first to incorporate advanced techniques from blockchains, as well as novel mechanisms, into such a cooperative decision making framework to design optimal strategies for honest participants. This framework allows various malicious behaviors and the maintenance of security and participant privacy. More specifically, we select a pool of validators who communicate to all participants, design a new consensus mechanism based on digital signatures for these validators, invent a UCB-based strategy that requires less information from participants through secure multi-party computation, and design the chain-participant interaction and an incentive mechanism to encourage participants' participation. Notably, we are the first to prove the theoretical regret of the proposed algorithm and claim its optimality. Unlike existing work that integrates blockchains with learning problems such as federated learning which mainly focuses on  optimality via computational experiments, we demonstrate that the regret of honest participants is upper bounded by $\log{T}$ under certain assumptions. The regret bound is consistent with the multi-agent multi-armed bandit problem without malicious participants and the robust multi-agent multi-armed bandit problem with purely Byzantine attacks which do not affect the entire system.
\end{abstract}

\section{Introduction}

Multi-armed Bandit (MAB) \citep{auer2002finite, auer2002nonstochastic} models the classical sequential decision making process that dynamically balances between exploration and exploitation in an online context. Specifically, in this paradigm, a player engages in a game, from which the player selects precisely one arm and observes the corresponding reward at each time step. The player aims to maximize the cumulative reward throughout the game. This is also equivalent to the so-called regret minimization problem navigating the trade-off between exploration (e.g., exploring unknown arms) and exploitation (e.g., favoring the currently known optimal arm). The recent emerging advancement of federated learning, wherein multiple participants jointly train a shared model, has spurred a surge of interest in the domain of  multi-agent multi-armed bandit (multi-agent MAB). In this context, multiple participants concurrently interact with multiple MABs, with the objective being the optimization of the cumulative averaged reward across all the participants through communications. Significantly, in addition to the exploration-exploitation trade-off, these participants engage in communication constrained by the underlying graph structure, which necessitates the exploration of the information of other participants and to develop strategies accordingly.

Numerous research has been focused on the multi-agent MAB problem, including both centralized settings as in \citep{bistritz2018distributed,zhu2021federated,huang2021federated,mitra2021exploiting,reda2022near,yan2022federated}, and decentralized settings as in \citep{landgren2016distributed,landgren2016distributed_2,landgren2021distributed,zhu2020distributed,martinez2019decentralized,agarwal2022multi}, where it is assumed that reward distributions are uniform among participants, namely homogeneous. Recent attention has shifted towards addressing decentralized, heterogeneous variants, including \citep{tao2022optimal, wang2021multitask, jiang2023multi, zhu2020distributed,zhu2021decentralized,zhu2021federated, zhu2023distributed, xu2023decentralized}, which are more general and bring additional complexities. In these scenarios, the shared assumption is that all participants exhibit honesty, refraining from any malicious behaviors, and adhere to both the shared objective and the designed strategies. However, real-world scenarios often deviate from this assumption, are composed of  malicious participants that perform disruptively. Examples include failed machines in parallel computing, the existence of hackers in an email system, and selfish retailers in a supply chain network. Consequently, recent research, such as \citep{vial2021robust}, has focused on the multi-agent MAB setting with malicious participants, which is formulated as a robust multi-agent MAB problem. This line of work yields algorithms that perform optimally, provided that the number of malicious participants remains reasonably limited. 
More recently, the work of \citep{zhu2023byzantine} proposes a byzantine-resilient framework and shows that collaboration in a setting with malicious participants upgrades the performance if at every time step, the neighbor set of each participant contains at least $\frac{2}{3}$ ratio of honest participants and downgrades the performance otherwise.


However, there are three major concerns related to the existing robust multi-agent MAB framework, namely optimality, security, and privacy, respectively. Firstly, in \citep{vial2021robust}, the truthfulness of the integrated reward estimators by participants is not taken into account. Every participant maintains reward estimations and thus we also call them estimators. In essence, it means it might not be possible to assert the correctness of these estimators, even though the relative differences between the arms are bounded. In certain scenarios, estimators play a crucial role in guiding decision making. For instance, in the context of smart home \citep{zhao2020privacy}, driven by the rapid growth of the Internet of Things (IoT), in a smart home device setting the suppliers of the devices are the participants monitored by the manufacturer, the devices are the arms, and the consumers are the environment, the manufacturer seeks to understand consumer behavior. The reward corresponds to any metric measuring consumer engagement. Each supplier develops its own engagement (reward) estimatation by arm pulls where it is important for the estimators to be accurate. The knowledge about the ground truth, i.e. consumer behavior in expectation across time, is essential, making the correctness of estimators a critical concern. 
Secondly, there is the possibility of malicious participants (suppliers) exhibiting various disruptive behavior beyond broadcasting inaccurate estimators. The attack behaviors may have an impact on the coordination mechanism where honest participants operate in, with the conflict between these two groups of participants (malicious and honest) which may ultimately disrupt the entire system. In this case, the honest participants cannot obtain any rewards and significantly downgrades the performance, which is a facet not covered in 
existing work~\citep{vial2021robust, zhu2023byzantine}. For example, in a network routing problem, where devices (i.e., participants) send information through communication channels that represent the arms to maximize information throughput (i.e., the reward), malicious participants could intentionally cause channel congestion and disrupt the traffic that honest participants rely on. This has the potential to systematically affect the performance of honest participants, which has not yet been studied,  serving as a significant motivation for this paper. Thirdly, existing literature assumes that participants are willing to share all the information of their interactions with other participants, including the number of pulls of arms and the corresponding reward estimators based on these number of pulls. This, however, exposes the participants to the risk of being less private, as it might be easy to retrieve the cumulative reward and action sequence, based on the shared information. These risks can lead to severe privacy leakage and thus need to be addressed, which has not yet been explored by the existing work and thereby motivating our work herein.


Notably, blockchains have a great potential to address these challenges, which are fully decentralized structures 
and have demonstrated exceptional performance in enhancing system security and accuracy across a wide range of domains \citep{feng2023cobc}. This trending concept, widely applied in finance, healthcare and edge computing, was initially introduced to facilitate peer-to-peer (P2P) networking and cryptography, as outlined in the seminal work by \citep{nakamoto2008bitcoin}. A blockchain (permissioned where the set of participants is fixed versus permisionless where the set of participants is dynamic and anyone can join) comprises of a storage system for recording transactions and data, a consensus mechanism for participants to ensure secure decentralized communication, updates, and agreement, and a verification stage to assess the effectiveness of updates, often referred to as block operations \citep{niranjanamurthy2019analysis}, which thus provides possibilities for addressing the aforementioned concerns. First, the existence of verification guarantees the correctness of the information before adding the block to the maintained chain, and the storage system ensures the history is immutable. 
Secondly, the consensus mechanism ensures that honest validators, which are representatives of participants, need to reach a consensus even before they are aware of each other's identities, leading to a higher level of security and mitigating systematic attacks. Lastly, enabling cryptography and full decentralization without a central authority has the potential to improve the privacy level. However, no attention has been given to understanding how blockchains can be incorporated into an online sequential decision making regime, creating a gap between multi-agent MAB and blockchains that we take a step to close it. 

There has been a line of work adapting blockchains into learning paradigms, and blockchain-based federated learning has been particularly successful as in \citep{li2022blockchain, zhao2020privacy, lu2019blockchain, wang2022blockchain}. In this context, multiple participants are distributed on a blockchain, and honest participants aim to optimize the model weights of a target model despite the presence of malicious participants. Notably, the scale of the model has led to the introduction of a new storage system on the blockchain, the Interplanetary File System (IPFS), which operates off-chain, ensuring the stability and efficiency of block operations on the chain. However, due to the fundamentally different natures of decision making in MAB and federated learning, the existing literature does not apply to the multi-agent MAB setting, which serves as the motivation for developing a novel formulation and framework for blockchain-based multi-agent MAB. Moreover, there is limited study on the theoretical effectiveness of blockchain-based federated learning, as most studies focus on their deployment performances. Theoretical validity is crucial to ensure cybersecurity because deploying blockchains, even in an experimental setting, is risky and has been extremely challenging.  Henceforth, it remains unexplored how to effectively incorporate blockchains into the robust multi-agent MAB framework, and how to analyze the theoretical properties of the new algorithms, which we address herein. 

Moreover, existing blockchain frameworks have limitations which make them not completely suitable for multi-agent MAB problems due to their online sequential decision-making nature. For instance, some consensus protocols assume the existence of a leader \citep{chen2017security}, introducing authority risks. Additionally, the general rule of practical Byzantine Fault Tolerant (PBFT) \citep{lei2018reputation} is to secure more than $\frac{2}{3}$ of the votes \citep{chen2019algorand}, which can be impractical in real-world scenarios. Meanwhile, it remains largely unclear about a theoretically efficient and effective validator selection protocol \citep{lei2018reputation} or a protocol where participants would be willing to participate \citep{nojoumian2019incentivizing, zhao2020privacy}. Most importantly, the online decision making problem necessitates the deployment of strategies offering real-time interactions with an exogenous environment, a feature not present in traditional blockchain frameworks. Consequently, the incorporation requires careful modifications to blockchains and the introduction of new mechanisms.  We address these challenges herein. 

To this end, we herein propose a novel formulation of robust multi-agent MAB incorporating blockchains. Specifically, we are the first to study the robust multi-agent MAB problem where participants are fully distributed, can be malicious, and operate on permissioned blockchains. In this context, a fixed set of participants pull arms and communicate to validators, and validators communicate with one another and decide on a block to be sent to the chain. participants can only receive rewards when the block containing the consensus information is approved by the chain, in order to ensure security at each time step, which differs largely from the existing MAB frameworks. In other words, the rewards are conditionally observable, even for the pulled arm, which complicates the traditional bandit feedback, as disapproved blocks introduce new challenges in this online and partial information setting. participants are allowed to be malicious in various disruptive aspects during the game. The objective of the honest participants is to maximize their averaged cumulative received reward. The blockchain keeps track of everything (the history is immutable), guarantees the functionality of the coordination mechanism through chain operations, and communicates with the environment. The proposed formulation introduces additional complexities, as participants not only design strategies for selecting arms but also strategically interact with both other participants and the blockchain. 

Additionally, we develop an algorithmic framework for the new formulation, motivated by existing work while introducing novel techniques. The framework uses a burn-in and learning period. We incorporate a UCB-like strategy into the learning phase to perform arm selection, while using random arm selection during the burn-in period. We also design a theoretically effective validator selection mechanism that eliminates the need for an authorized leader, including both full decentralization and efficient reputation-based selection. We propose the update rules for both participants and validators to leverage the feedback from both the environment and the participant set. Furthermore, we modify the consensus protocol without relying on $\frac{2}{3}$ voting; instead, we use a digital signature scheme \citep{goldwasser1988digital} coupled with the consensus protocol in \citep{lamport2019byzantine}. Moreover, we introduce the role of a smart contract \citep{hu2020comprehensive} that enables interaction with both the blockchain and the environment, which validates the consensus information and collects the feedback from the environment. To incentivize the participation of malicious participants in the game (we want the malicious participants to actively participate via information sharing in order to be identified soon) we invent a novel cost mechanism inspired by the recent use of mechanism design in federated learning \citep{murhekar2023incentives}. It is worth noting that the existence of this smart contract and cost mechanism also guarantees the correctness of the information transmitted on the chain.

Subsequently, we perform theoretical analyses of the proposed algorithm. More specifically, we formally analyze the regret that reflects optimality and fundamental impact of malicious behavior on blockchains. Precisely, we show that under different assumptions in different settings, the regret of honest participants is always upper bounded by $O(\log T)$, which is consistent with the existing algorithms for robust multi-agent MAB problems \citep{zhu2023byzantine, vial2021robust}. This is the very first theoretical result on leveraging blockchains for online sequential decision making problems, to the best of our knowledge.  Furthermore, this regret bound coincides with the existing regret lower bounds in multi-agent MAB when assuming no participants are malicious \citep{xu2023}, implying its optimality. Moreover, we find that,  surprisingly, various aspects about security are by-products of optimality, since the blockchain framework needs to be secure enough in order to maximize the received reward. 

Our main contributions are as follows. First, we propose a novel formulation of multi-agent MAB with malicious participants, where rewards are obtainable only when the coordination mechanism's security is guaranteed. Additionally, the actual received rewards account for the accuracy of the shared information through our proposed cost mechanism. To maximize the cumulative rewards of honest participants, we develop a new algorithmic framework that introduces blockchain techniques. Along the way we design new mechanisms and protocols. We also prove the theoretical effectiveness of the algorithm through an extensive analysis of regret under mild assumptions on the problem setting, such as the ratio of honest participants, the cost definition, and the validator selection protocol. Ultimately, this work comprehensively bridges the gap between cybersecurity and online sequential decision making.

The structure of the paper is as follows. In Section 2, we introduce the problem formulation, along with the notations that are used throughout. Following that, in Section 3, we propose the algorithmic framework. Subsequently, in Section 4, we provide detailed analyses of the theoretical regret guarantee for the proposed algorithms. Meanwhile, in Section 5, we provide a discussion on other performance measure, such as security and equilibrium. Lastly, we conclude the paper and point out possible future directions in Section 6. 



\section{Problem Formulation}
We start by introducing the notations used throughout the paper. Consistent with the traditional MAB setting, we consider $K$ arms, labeled as $1,2, \ldots, K$. The time horizon of the game is denoted as $T$, and let us denote each time step as $1 \leq t \leq  T$. Additionally as in Multi-agent MAB, let us denote the number of participants as $M$ labeled from $1$ to $M$. We denote the public and secret keys of participant $m$ as $(PK_m, SK_m)$ for any $1 \leq m \leq M$. The list of public keys ${PK_1, PK_2, \ldots, PK_M}$ is public to anyone, in the order indicated by the participant set. 
Meanwhile, in our newly proposed blockchain framework, we denote the total number of blocks as $B = T$ and whether each block at time step $t$ is approved or not is represented by a binary variable $b_t \in \{0,1\}$. Let us denote the reward of arm $i$ at participant $m$ at time step $t$ as $\{r_i^m(t)\}_{i,m,t}$, which follows a stochastic distribution with a time-invariant mean value $\{\mu_i\}_{i}$. Let $a_m^t$ be the arm selected at time $t$ by participant $m$ and let $n_{m,i}(t)$ be the number of arm pulls for arm $i$ at participant $m$ at time $t$. We denote the set of honest participants and malicious participants as $M_H$ and $M_A$, respectively, which are not known apriori. Note that they are time-invariant. Similarly, let $S_V(t)$ denote the set of validators at time $t$ which is algorithmically determined. We denote the estimators maintained at participant $m$ as $\bar{\mu}_{i}^{m}(t), \Tilde{\mu}_{i}^{m}(t)$ for local and  global reward estimators, respectively, and the validators estimators as $\Tilde{\mu}_{i}(t)$. We point out that $\Tilde{\mu}_{i}(t)$ is a function of $\bar{\mu}_{i}^{m}(t)$. 

Let $h_j^t(\mathcal{F}_t)$ be the estimators given by malicious participant $j \in M_A$ where $\mathcal{F}_t$ denotes the history up to time step $t$ (everything on the blockchain and additional information shared by other participants). The blocks on the blockchain record the execution information. Specifically, at each time step $t$ a block  records the global estimators $\{\Tilde{\mu}_i(t)\}_{i}$ and local estimators $\{\bar{\mu}_i^m(t)\}_{m,i}$, encrypted values of $\{n_{m,i}(t)\}_{m,i}$, $B_t$ specified in Aggregation, and arms $a_m^t$ pulled. Moreover, the block also records the reward  $r_i^m(t)$ of each honest participant $m \in M_H$. The information related to an individual participant, such as $\bar{\mu}_i^m(t)$ and $r_i^m(t)$, is signed by the participants using digital signatures that are the same across time. By using Global Update in Algorithm 4, the definition of  $\{\Tilde{\mu}_i(t)\}_{i}$  and (\ref{eq:xy}), all these quantities can be verified given $r_i^m$ and $a_m^t$. 


The process during one iteration is as follows. At the beginning of each decision time, each participant selects an arm based on its own policy. Then, a set of validators is selected, and the participants broadcast their reward estimators to the validators. The validators perform aggregation of the collected information. Next, they run a consensus protocol to examine whether the majority agree on the aggregated information, a process called validation. They send the validated information to the smart contract, which verifies its correctness and sends feedback to the environment. If the smart contract is approved, the blockchain is updated.  Lastly, the environment distributes the reward information plus cost based on the feedback from the smart contract (only if the block has been approved). The participants then update their estimators accordingly. The corresponding flowchart is presented in Figure~\ref{fig:1}.

\begin{figure}[h]
    \centering
\scalebox{0.45}{\includegraphics{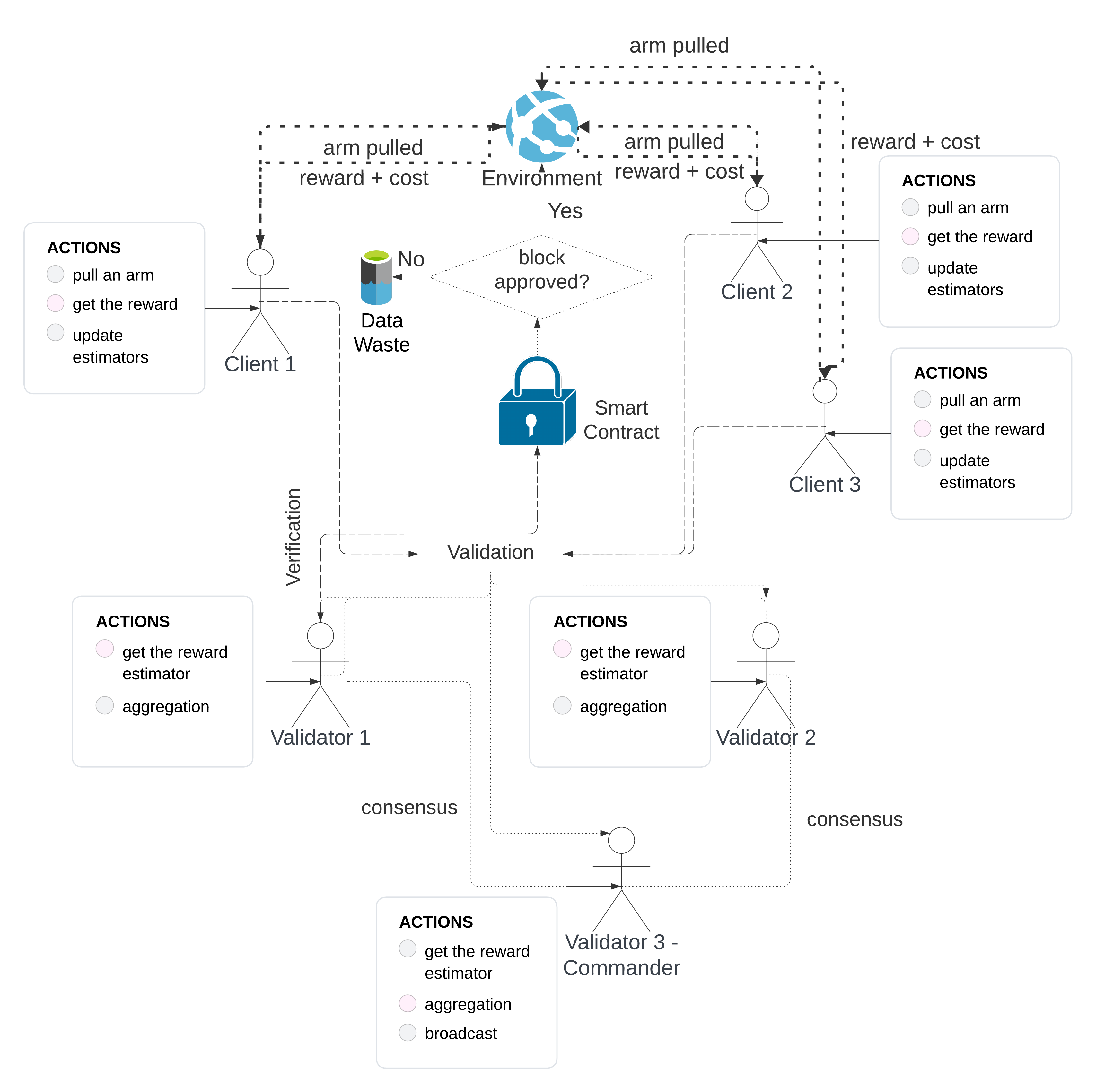}}
\caption{The flow of the algorithm}\label{fig:1}
\end{figure}

\paragraph{Cost Mechanism} We propose a cost mechanism where if the estimators from the malicious participants are used in the validated estimators, i.e.
\begin{align*}
   \frac{\partial \Tilde{\mu}_{i}(t)}{\partial \bar{\mu}_{i}^m(t)} \neq0,  
\end{align*} then the honest participants incur a cost of $c_t^m \geq 0$ and malicious participants receive $c_t^m < 0$, which they are not aware of until the end of the game. It incentives the participation of malicious participants, in particular, given that they may not be willing to share anything. In the meantime, as a by-product, it also penalizes the aggregated estimators by the honest participants, which ensures the correctness of the estimators. In addition $c_t^m = 0$ for every $m$ if $\frac{\partial \Tilde{\mu}_{i}(t)}{\partial \bar{\mu}_{i}^m(t)} = 0$.

We also present some terminology used in the digital signature scheme \citep{goldwasser1988digital} and secure multi-party computation \citep{asharov2012multiparty}. 
\paragraph{Existential Forgery} Following the definition in ~\citep{goldwasser1988digital}, malicious participants successfully perform an existential forgery if there exists a pair consisting of a message and a signature, such that the signature is produced by an honest participant.

\paragraph{Adaptive Chosen Message Attack} Consistent with ~\citep{goldwasser1988digital}, we consider the most general form of a message attack, namely the adaptive chosen message attack. In this context, a malicious participant not only has access to the signatures of honest participants but also can determine what message to send after seeing these signatures. This grants the malicious participant a high degree of freedom, thereby making the attack more severe.  

\paragraph{Universal Composability Framework} For homomorphic encryption, more specifically, secure multi-party computation, we follow the standard framework as in ~\citep{canetti2001universally}. Specifically, an exogenous environment, also known as an environment machine, interacts sequentially with a protocol. The process runs as follows. The environment sends some inputs to the protocol and receives outputs from the protocol that may contain malicious components. If there exists an ideal adversary such that the environment machine cannot distinguish the difference between interacting with this protocol or the ideal adversary, the protocol is deemed universally composable secure.

With the goal to maximize the total cumulative (expected) reward of honest participants, we define the regret as follows. We denote the cumulative reward of honest participants as 
\begin{align*}
r_T = \sum_{m \in M_H}\sum_{t=1}^Tr_{a_m^t}^{m}(t)1_{b_t = 1} - \sum_{t=1}^Tc_t
\end{align*}
and the regret as
\begin{align*}
    R_T = \max_i\sum_{m \in M_H}\sum_{t=1}^Tr_{i}^{m}(t)1_{b_t = 1} - r_T
\end{align*}
and pseudo regret
\begin{align*}
    \bar{R}_T = \max_i\sum_{m \in M_H}\sum_{t=1}^T\mu_{i}^m - E[r_T].
\end{align*}

We argue the rationale of this regret definition as follows. It holds true that these two regret measures are well-defined, considering that $M_H$ is fixed and does not change with time. Furthermore, our definition aligns with those used in the context of blockchain-based federated learning \citep{zhao2020privacy}, as their objective is to optimize the model maintained by honest participants, though without involving online decision making. Additionally, this definition is consistent with the existing robust multi-agent MAB problem \citep{vial2021robust}, except that the cost mechanism is introduced which incentives participation and guarantees correctness. Compared to the multi-agent MAB, our regret is averaged over only honest participants due to the existence of malicious participants. Note that the two measures are the same if the number of malicious participants is zero since the cost $c_t$ is also zero in such a case, implying consistency. 

\section{Methodologies}

In this section, we present our proposed methodologies within this new framework.  
Notably, we develop the first algorithmic framework at the interface of blockchains in cybersecurity and multi-agent MAB in online sequential decision making, addressing the joint challenges of optimality, security, and privacy. We leverage the blockchain structure while introducing new advancements to the existing ones, to theoretically and efficiently guarantee the functionality of the chain  with new consensus protocols and a cost mechanism. 
Additionally, it is designed for online sequential decision making scenarios, distinguishing our work from existing literature on federated learning. Moreover, compared to existing work on Byzantine-resilient multi-agent MAB, our methodology operates on a blockchain with an added layer of security and privacy. 

More specifically, the algorithmic framework is composed of two phases: the burn-in period, which is a warm-up phase for $t \leq L$, where $L$ is the length of the burn-in period, and the learning period, where $t > L$. It consists of 5 functions, with the main algorithm presented in Algorithm \ref{alg:bc}, and the remaining functions detailed in Algorithms 2-5. Algorithm \ref{alg:bc} constitutes the core of the methodology, including the sequential strategies executed by the honest participants, black-box operations by the malicious participants, and the chain executions. 

The core algorithm includes several stages, as indicated in the following order.  
We present the pseudo code of the core algorithm in Algorithm \ref{alg:bc}, named BC-UCB. Here, the common random seed $\bar{q}_t$, $t = 1, \ldots, T$ and the random seed for each participant $q = (q_1, q_2, \ldots, q_M)$ are publicly known in advance. Function $VRF$ refers to verifiable random functions proposed in \citep{micali1999verifiable} composed of $G$ that represents the generating function for the public and secret key with seed $q^0$, i.e.  $G(q^0) = (pk, sk)$, and $VRF_F(\bar{q}_t,G(q^0)) = (hash, \pi)$ where $hash$ is a hash value and $\pi$ is a function (proof) that returns $True$ or $False$ given $hash$ and public key $pk$, i.e. $\pi(hash, pk)$ outputs $True$ or $False$.  
Let $hl$ be the size of $hash$ which is an input to $\pi$. Let $B(k, n, w)$ be the probability of selecting $k$ samples from $n$ samples with a success probability $w$ where $k \leq n$. This is also known as the probability density function of the binomial distribution with success probability $w$. For any multi-set $S_0$, $majority(S_0)$ refers to an element in $S_0$ with the highest count.

\paragraph{Arm selection} As in an MAB framework, the participants decide which arm to pull at each time step. The strategies depend on whether participants are honest or malicious. The honest participants follow a UCB-like approach. More specifically, each honest participant $m$ selects arm $a_m^t = t \mod K$ during the burn-in period. During the learning period, it assigns a score to each arm $i$ and selects the arm with the highest score, which can be formally written as
\begin{align*}
    a_m^t = \argmax_i\Tilde{\mu}_{i}^m(t) + F(m,i,t)
\end{align*}
where $\Tilde{\mu}_{i}^m(t) $ is the maintained estimator at participant $m$. Here $F(m,i,t) = (\frac{C_1\log t}{n_{m,i}(t)})^{\beta}$ with constant $C_1, \beta$ being specified in the theorems. A malicious participant $j$, however, selects arms based on arbitrary strategies, which is also known as Byzantine's attack and  written as 
\begin{align*}
    a_j^t = h^t_j(\mathcal{F}_t) \in [K]
\end{align*}
where $\mathcal{F}_t$ denotes the history up to time step $t$ (everything on the blockchain and additional information shared by other participants).

\paragraph{Validator or Commander Selection}
At each time step, a coordination mechanism or iterative protocol selects a pool of participants allowed to act on the chain, known as validators. In each iteration of the protocol, first the commander is selected which communicates with all validators and in the rest of the iteration the validators communicate among themselves. The commander can change from one iteration to the next one. Specifically, the coordination mechanism samples the set of validators and commanders according to Algorithm \ref{vc}, based on the trust coefficients of participants $p_m(t), w_m(t)$, which are set to 1 initially. 
The chain relays this set of validators to aggregate the reward estimators and to achieve consensus as detailed below. 

We use a smart contract that takes membership of a participant in $S_V(t)$ as input and produces the sorted set $S_V(t)$ based on the public keys $PK_m$ of participants. It is worth noting that the sorting function can be incorporated into the script, either by loading packages in the programming script using existing programming languages or by implementing the topological sorting procedure \citep{dickerson2017adding}. Then, the participants access this smart contract $sc_{sort}$ with input $S_V(t), PK$ to obtain the set of validators $S_V(t)$ from its output.

\paragraph{Broadcasting} During broadcasting, the participants sent information to validators which then perform the aggregation step. To expand, malicious participant $j$ broadcasts its estimators $\bar{\mu}_{i}^j(t)$ to the validators using a black-box attack, e.g. a Byzantine's attack or a backdoor attack. Honest participant $m$ broadcasts its true reward estimators $\bar{\mu}_{i}^m(t)$ to the validators.

\paragraph{Aggregation} Next, the validators integrate the received information. Specifically, for each honest validator $j$, an honest validator determines the set, $A_t^j, B_t^j$ as follows.  

For $t > L$, the set $A_t^j$ reads as 
\begin{align*}
    m \in A_t^j \Leftrightarrow n_{m,i}(t) > \frac{n_{j,i}(t) }{k_i(t)} \text{ for every $i$}
\end{align*} 
where $k_i(t) \geq \max_{k \in M}\frac{n_{k,i}(t)K}{L}$ is the threshold parameter  which can be constructed through the secure multi-party computation protocol as in \citep{asharov2012multiparty}, without knowing the value of $n_{m,i}(t)$ to ensure privacy. More specifically, each participant $m$ sends $n_{m,i}(t)$ and the value of $k_i(t)$ to the protocol. The protocol then outputs whether $m \in A_t^j$. 

The set $B_t^j$ is computed as follows, depending on the size of $A_t^j$. If $|A_t^j| > 2f$ where $f = |M_A|$ and the process is in the learning period $t > L$, then $B_t^j = \cup_{i}\{(m,\bar{\mu}_{i}^m(t)): \text{$\bar{\mu}_{i}^m(t)$ is smaller than the top $f$ values in $A_t^j$ and larger than the bottom $f$ values in $A_t^j$}\}$.

Otherwise in burn-in, $B_t^j = \{t \text{ mod } K\}$ and $A_t^j = \emptyset$. 


Once again, the malicious participants choose the sets $A_t$ and $B_t$ in a black-box manner.

\paragraph{Consensus} The consensus protocol is central to the execution of the blockchain and guarantees that the chain is secure. More specifically, we incorporate the digital signature scheme \citep{goldwasser1988digital} into the solution to the Byzantine General Problem \citep{lamport2019byzantine} under any number of malicious validators. The pseudo code is presented in Algorithm \ref{co}. First, a commander is selected from the validators that broadcasts $B_t$ to other validators with its signature generated by \citep{goldwasser1988digital}, which we call a message. This process is then repeated at least $M$ times, based on the algorithm in \citep{lamport2019byzantine}. The validators output the mode of the maintained messages. The consensus is successful if more than 50\% of the validators output the $B_t$ maintained by the honest validators. Otherwise, the consensus step fails, resulting in an empty set 
 $B_t$. 

\paragraph{Global Update}
The set $B_t$ is then sent to the validators, which computes the average of the estimators within $B_t$, known as the global update detailed in Algorithm \ref{gu}. More precisely, for each arm $i$ at time step $t$, the estimator is computed as 
\begin{align*}
& \Tilde{\mu}_i(t) = \frac{1}{2}(\hat{\mu}_{i}(t) +\Tilde{\mu}_i(\tau))  \\
& \hat{\mu}_{i}(t) = \frac{\sum_{m \in B_t}\bar{\mu}_{i}^m(t)}{|B_t|}
\end{align*}
where $\tau = \max_{s<t}\{b_s = 1\}$. 

if $B_t$ is not empty,
and 
\begin{align*}
   \Tilde{\mu}_{i}(t)  =  \infty, \hat{\mu}_{i}(t) = \infty, 
\end{align*}
otherwise. 

\paragraph{Block Verification}
The validators run the smart contract $sc_{block}$ to validate the block and assigns $b_t=1$ if the estimator satisfies the condition 
\begin{align*}
    \Tilde{\mu}_{i}(t)  \leq 2. 
\end{align*}
It disapproves the block otherwise, which is denoted as $b_t = 0$.

\paragraph{Block Operation} At the beginning of the algorithm, the environment sets a random cost value $c_t =c$ between $0$ and $1$. The smart contract sends the output containing the validated estimator $\Tilde{\mu}_i(t)$, the set $B_t$, and the indicator $b_t$ of whether the block is approved to the environment. Subsequently, the environment determines the rewards, namely Block Operation, as in Algorithm \ref{or}, to be distributed to the participants based on the received information from the smart contract, in the following three cases.

Case 1: If $b_t =1$ and $B_t \subset M_H$, i.e. 
\begin{align*}
   \frac{\partial \Tilde{\mu}_{i}(t)}{\partial \bar{\mu}_{i}^m(t)} = 0 \text{ for every } m \not\in M_H,  
\end{align*} then the environment distributes $r_{a_m^t}^m(t)$ and $\Tilde{\mu}_{i}(t)$ to participant $m$ for any $1 \leq m \leq M$. 

Case 2:
If $b_t =1$ and $B_t \cap M_H < |B_t|$, i.e. there exists $m \not\in M_H$ such that 
\begin{align*}
   \frac{\partial \Tilde{\mu}_{i}(t)}{\partial \bar{\mu}_{i}^m(t)} \neq 0, 
\end{align*} then 
the environment distributes $r_{a_m^t}^m(t) -c_t$ and $\Tilde{\mu}_{i}(t)$ to any honest participant $m$, and $r_{a_j^t}^j(t) + c_t$ to any malicious participant $j$.

Case 3: If $b_t =0$, the environment distributes nothing to the participants. 

\paragraph{Participants' Update} After receiving the information from the environment, the honest participants update their maintained estimators as follows. 

\paragraph{Rule} For the global reward estimator $\Tilde{\mu}_{i}^m(t)$, they update it when they receive $\Tilde{\mu}_{i}(t)$, i.e. 
\begin{align*}
    \Tilde{\mu}_{i}^m(t) = \Tilde{\mu}_{i}^m(t),
\end{align*}
and otherwise,
\begin{align*}
    \Tilde{\mu}_{i}^m(t) = \bar{\mu}_{i}^m(t).
\end{align*}
\indent For the number of arm pulls and the local reward estimators, they update them as 
\begin{align}\label{eq:xy}
    & n_{m,i}(t) = n_{m,i}(t-1) + 1_{b_t = 1}\cdot1_{a_m^t = i} \notag \\
    & \bar{\mu}_i^m(t) = \frac{\bar{\mu}_i^m(t-1) +  r_{a_m^t}^m(t) \cdot 1_{a_m^t = i}}{n_{m,i}(t)}. 
\end{align}

\indent For the trust coefficients, we let the trust coefficients $p_m(t)$ and $w_m(t)$ be updated based on formulas given later.  

\section{Regret Analysis}

In this section, we demonstrate the theoretical guarantee of our proposed framework by conducting a comprehensive study of the regret defined for the honest participants in Section 2. Specifically, to allow for flexibility and generalization, we consider various problem settings, including the number of malicious participants $M_A$, the cost definition $c_t$, the commander selection rule, and the validator selection rule.

{\LinesNumberedHidden
\begin{algorithm}[H]
\SetAlgoLined
\caption{BC-UCB}\label{alg:bc}
 Initialization:  For participants $1,2, \ldots, M$,  arms $1, \ldots, K$, at time step $0$ we set $\Tilde{\mu}^{m}_i(0) = \hat{\Tilde{\mu}}_i(0) =  n_{m,i}(1) = 0$; the number of honest participants $M_H$; Verifiable Random Function $VRF$\\
 \For{$t = 1, 2, \ldots,T$}{
    \For(\tcp*[f]{Validator/Commander Selection}){each participant $m$}{Sample $z =$ SELECTION($t, m, p_m(t), VRF$). If $z=1$, participant $m$ is a validator. \par
    Sample $z =$ SELECTION($t, m, w_m(t), VRF$). If $z=1$, participant $m$ is a commander.}
Let $S_V(t)$ be the set of all validators and $S_C(t)$ be the set of commanders.  \par
\For(\tcp*[f]{Arm Selection - UCB}){each participant $m \in M_H$}{
\eIf{$k \in A_t^m$ for every $k\in M_H$ by using secure multi-party computation with $S_V(t)$ and $S_C(t)$ and $t > L$}{ $a_m^t = \argmax_i \Tilde{\mu}_{i}^m(t) + F(m,i,t)$}
{ Sample an arm $a_m^t = t \mod K$.
}
Pull arm $a_m^t$\\
}
\For(\tcp*[f]{Arm Selection - Any Strategy}){each participant $m \not\in M_H$}{
Select an arm $a_m^t$ and pull arm $a_m^t$\\
}
\For(\tcp*[f]{Broadcasting}){each participant $m$}{
Broadcast $\bar{\mu}^m_i(t)$ to validators $S_V(t)$, where malicious participants $m \not\in M_H$ use an attack regarding an arm $a_m^t$, i.e., $\bar{\mu}^{m}_i(t) = \bar{h}_{m,i}^t(\mathcal{F}_t)$.
}
\For(\tcp*[f]{Aggregation}){each participant $m \in S_V(t)$}{
Validator $m \in M_H \cap S_V(t)$ determines the set $B_t^m = B_t$ containing trusted participants $j$ and the corresponding estimators $\bar{\mu}^{j}_i(t)$  \par
 Validator $m \not\in M_H, m \in S_V(t)$ arbitrarily determines the set $B_t^m$ 
 }
\tcp{Consensus} \par 
Validators run consensus on $\{B_t^m\}_{m}$ according to CONSENSUS($S_C(t), \{B_t^m\}_m, M$)  \par
The validators run the smart contract $sc_{block}$ to compute $\Tilde{\mu}_i(t)$ according to GLOBAL\_UPDATE($B_t$) \tcp*[f]{Global Update} \par
Block Validation: \tcp*[f]{Block Verification} \par 
\eIf{there exists $i \in \{1, 2, \ldots, K\}$ with global estimator $\Tilde{\mu}_i(t) < \infty$}{Approve the block by letting $b_t = 1$}{Disapprove the block by letting $b_t = 0$} 
\tcp{Environment} \par 
The environment sends rewards to participants using OPERATION($\Tilde{\mu}_i(t), {a_m^t}_m,B_t, b_t$) \par 
\For(\tcp*[f]{participants' Update}){each participant $m$}{participant $m \in M_H$ updates $\Tilde{\mu}_m(t), n_{m,i}(t), \bar{\mu}_m(t), p_m(t), w_m(t)$ based on Rule; 
participant $m \not\in M_H$ updates $\Tilde{\mu}_m(t), n_{m,i}(t), \bar{\mu}_i^m(t)$ arbitrarily and updates $p_m(t), w_m(t)$ based on Rule}
 }
 \end{algorithm}
}

\begin{algorithm}
\SetAlgoLined
\caption{Validator or Commander Selection}\label{vc}
\begin{algorithmic}[1]
\Function{Selection}{$t, m, l, VRF$}
\State Let $(pk_{m}, sk_{m}) = G(q_m)$ 
\State Let  $(hash, \pi) = VRF_F(\bar{q}_t, pk_{m}, sk_{m})$ 
\State $z=0$ \;
\State \textbf{if} $\frac{hash}{2^{hl}} \not\in [0, B(0,1,l)]$ \textbf{then} 
\State \indent $z = 1$
\State \textbf{end}
\State If $\pi(hash,pk_m) = True$ then \textbf{return $z$} else \textbf{return 0}\;
\EndFunction
\end{algorithmic}
\end{algorithm}

\begin{algorithm}
\caption{Consensus}\label{co}
\begin{algorithmic}[1]
\Function{Consensus}{$S_C(t), \{B_t^m\}_m, M$}
\State Run $sc_{sort}(S_V(t), PK)$ which returns sorted $S_C(t)$
    \State \textbf{for} $h = 1, 2, \ldots, |S_C(t)|$ \textbf{do}
    \State \indent Generate the digital signature $\{s_h^m\}_m$ as in \citep{goldwasser1988digital}\;
   \State \indent Define a message as $(s_h^m, B_t^m)$\;
   \State \indent Execute Algorithm $SM(M)$ in \citep{lamport2019byzantine} with $S_C(t)[h]$ as the commander\;
   \State \indent Derive the received information $\Tilde{B}_t^h$ from $S_C(t)[h]$\;
   \State \indent and $v_t^m = 1$ if $\Tilde{B}_t^h= B_t^m$ at honest participant $m$ and 0 otherwise
   \State \If{$majority(v_t^m) = 1$}
   {Consensus is achieved and $B_t = B_t^m $}
    \State \Else 
    {Consensus fails and $B_t = \emptyset$}
    \State \Return $B_t$
\EndFunction
\end{algorithmic}
\end{algorithm}

\begin{algorithm}
    \caption{Global Update}\label{gu}
\begin{algorithmic}[1]
\Function{global\_update}{$B_t$}
    \State \If{$B_t$ is not empty} {Compute $\Tilde{\mu}_{i}(t) = \frac{\sum_{m \in B_t}\bar{\mu}_{i}^m(t)}{|B_t|}$  for each $i \in \{1, \ldots, K\}$}
    \State \Else
        {$\Tilde{\mu}_{i}(t) = \infty$ for each $i \in \{1, \ldots, K\}$}
    \State \Return $(\Tilde{\mu}_{i}(t))_{i \in \{1, \ldots, K\}}$
\EndFunction
\end{algorithmic}
\end{algorithm}

\begin{algorithm}
\caption{Operation}\label{or}
\begin{algorithmic}[1]
\Function{Operation}{$\{\Tilde{\mu}_i(t)\}_{i \in \{1, \ldots, K\}}, \{a_m^t\}_m, B_t, b_t$}
    \State \If{$b_t = 1$ and $B_t \subset M_H$} {Distribute $r_{a_m^t}^m$ and $\Tilde{\mu}_i(t)$ for every $i$ to every participant $m$} 
    \State \If{$b_t = 1$ and $B_t \cap M_H < |B_t|$} 
 {Distribute $r_{a_m^t}^m -c_t$ and $\Tilde{\mu}_i(t)$ for every $i$ to every honest participant $m \in M_H$ \par Distribute $r_{a_m^t}^m + c_t$ and $\Tilde{\mu}_i(t)$ for every $i$ to every malicious participant $m \in M_A$}
    \State \Else 
    {Distribute nothing to all participants }
    \State \Return 
\EndFunction
\end{algorithmic}
\end{algorithm}

\subsection{$M_A \leq \frac{1}{3}M$ with $\frac{1}{3}M +1 $ commanders and constant cost}

First, we consider the case when the number of honest participants is larger than $\frac{2}{3} \cdot M$. The cost mechanism uses constant cost, i.e. $c_(t) = c$, which requires honest participants to exclude any estimator that is from the malicious participants when updating $\hat{\mu}_i(t)$. Meanwhile, the commander selection makes sure that at least one honest participants serves as a commander, which allows the honest participants to achieve consensus on the accurate $\Tilde{\mu}$. The formal statement in this setting reads as follows.

\begin{theorem}
Let us assume that the total number of honest participants is at least $\frac{2}{3}M$. Let us assume that there is at least one honest participant in the validator set.  Meanwhile, let us assume that the malicious participants perform existential forgery on the signatures of honest participants with an adaptive chosen message attack. Lastly, let us assume that the participants are in a standard universal composability framework when constructing $A$. Then we have that 
\begin{align*}
    E[R_T|A] \leq (c+1)\cdot L + \sum_{m \in M_H}\sum_{k=1}^K\Delta_k([\frac{4C_1\log T}{\Delta_i^2}] + \frac{\pi^2}{3}) + |M_H|Kl^{1-T}
\end{align*}
where 
$L$ is the length of the burn-in period of order $\log{T}$, $c > 0$ is the cost, $C_1$ meets the condition that $\frac{C_1}{6|M_H|k_i\sigma^2} \geq 1$, $\sigma^2 \geq \frac{1}{M_H}$, $\Delta_i$ is the sub-optimality gap, $l$ is the length of the signature of the participants, and $k_i$ is the threshold parameter used in the construction of $A_t$. Here the set $A$ is defined as $A= \{\forall 1 \leq t  \leq T, b_t =1\}$ which satisfies that $P(A) \geq 1 - \frac{1}{l^{T-1}}$. 
\end{theorem}

\begin{proof}[Proof sketch] 

The full proof is provided in Appendix; the main logic is as follows. We decompose the regret into three parts: 1) the length of the burn-in period, 2) the gap between the rewards of the optimal arm and the received rewards, and 3) the cost induced by selecting the estimators of the malicious participants. For the second part of the regret, we bound it in two aspects. First, we analyze the total number of times rewards are received, i.e., when the block is approved, which is of order $1 - l^{1-T}$. Then, we control the total number of times sub-optimal arms are pulled using our developed concentration inequality for the validated estimators sent for verification. Concerning the third part, we bound it by analyzing the construction of $B_t$, which depends on the presence of malicious participants in $A_t$. By demonstrating that $A_t$ contains only a small number of malicious participants in comparison to the total number of honest participants, we show that $B_t$ does not induce additional cost. Combining the analysis of these three parts, we derive the upper bound on regret.

\end{proof}

\begin{remark}
It is worth noting that the minimum number of honest participants is consistent with \citep{zhu2023byzantine}. Although they establish the regret bound in a cooperative bandit setting with Byzantine attacks for any number of participants, the regret is only smaller than the individual regret when this assumption holds for every neighbor set of every honest participant at each time step. Otherwise, the regret is even larger, providing no advantage or motivation for participants to collaborate, essentially reducing the problem to the single-agent MAB problem. 
\end{remark}

\subsection{$M_A \leq \frac{1}{2}M$ with $\frac{1}{2}M +1 $ commanders and distance-based cost}
Along the line of work on robust optimization \citep{dong2023byzantine}, the common assumption is that at least $\frac{1}{2}$ participants are honest. To this end, we relax the assumption on the minimal number of honest participants from $\frac{2}{3}$ to $\frac{1}{2}$, while making modification to the definition of cost.

In this regard, we propose the following algorithmic changes. 

It is hard to analyze the proposed aggregation step. For the purpose of the analysis, we consider the following alternatives. 

\textbf{Option 2}

Construct a filter list $A_t$ such that 
\begin{itemize}
    \item if participant $m$ in $A_t$
    \begin{center}
for any $t > L$ and any $1 \leq j \leq M$, $k_i \cdot n_{m,i}(t) \geq  n_{j,i}(t)$
\end{center}
\end{itemize}

Construct a block list $B_t \subset A_t $ such that 
\begin{itemize}
    \item assume that $L$ is the length of the burn-in period
    \item if $m$ in $B_t$ 
    \begin{center}
for any $t > L$, \\
\{$m$: \text{$\bar{\mu}_{i}^m(t)$ is smaller than the top $f$ values and larger than the below $f$ values}\}
\end{center}
\end{itemize}

\textbf{Option 3}
Construct a filter list $A_t$ such that 
\begin{itemize}
    \item if participant $m$ in $A_t$
    \begin{center}
for any $t > L$ and any $1 \leq j \leq M$, $k_i \cdot n_{m,i}(t) \geq  n_{j,i}(t)$
\end{center}
\end{itemize}

Construct a block list $B_t \subset A_t $ such that 
\begin{itemize}
\item if $m$ in $B_t$ 
    \begin{center}
for any $t > L$, \\
\{$m$: \text{$\bar{\mu}_{i}^m(t)$ is smaller than the top $f$ values and larger than the below $f$ values}\}
\end{center}
\end{itemize}

If honest participant $m$ is a validator at time step $t$, then it maintains a blocklist $D_t$ such that 
\begin{itemize}
    \item participant $d \in D_t$ if all of the conditions are met 
    \begin{itemize}
        \item participant $d$ is a commander, and participant $d$ attacks to the consensus in that  $d$ signs two different messages
           and sends the message that is different from the one 
    \end{itemize}
\end{itemize}

Let $B_t$ =  $B_t \cap (D_t)^c$ where $(D_t)^c$ represents the compliment set of $D_t$.

We call the original one as Option 1. The choice of the option affects step 2 in Aggregation.

Let us assume that set $A_t, B_t$ are constructed based on Option 2, instead of Option 1 in Theorem 1.
    
\begin{remark}
    The global estimator $\Tilde{\mu}_i(t)$ in Global Update is constructed as 
    \begin{align*}
     \Tilde{\mu}_i(t) = P_t\Tilde{\mu}_i(t-1) + (1-P_t)\hat{\mu}_i(\tau)
    \end{align*}
\end{remark}
where $P_t = 1 - \frac{1}{t}$ and again $\tau = \max_{s < t}\{b_s = 1\}$. 

\begin{remark}
    The newly established estimator $\hat{\mu}_i(t)$ is constructed as
    \begin{align*}
        \hat{\mu}_i(t) = \frac{\sum_{m \in B_t}\bar{\mu}_i^m(t)}{|B_t|}
    \end{align*}
\end{remark}

\begin{remark}
    The distance metric associated with the global estimator is constructed as 
    \begin{align*}
       Dist(\Tilde{\mu}_i(t), \mu_i)  = | \Tilde{\mu}_i(t)- \mu_i|^6
    \end{align*}
    and this order of $6$ is tight based on the analysis. 
\end{remark}

\begin{remark}
    The cost associated with the global estimator is constructed as 
    \begin{align*}
        c_t = \min_i Dist(\Tilde{\mu}_i(t), \mu_i)
    \end{align*}
\end{remark}

\begin{remark}
    The decision rule needs the following modification. Specifically, the length of the exploration is now $(\frac{\log{t}}{n_i(t)})^{\frac{1}{6}}$. 
\end{remark}

Subsequently, we derive the corresponding regret guarantee under the more general assumption. The formal statement reads as follows. 

\begin{theorem}
Let us assume that the total number of honest participants is at least $\frac{1}{2}M$. 
Let us assume that there is at least one honest participant in the validator set.  Meanwhile, let us assume that the malicious participants perform existential forgery on the signatures of honest participants with an adaptive chosen message attack. Lastly, let us assume that the participants are in a standard universal composability framework when constructing $A$. Then we have that 
\begin{align*}
    E[R_T|A] \leq (c+1)\cdot L + O(\log{T}) + |M_H|Kl^{1-T}
\end{align*}
where 
$L$ is the length of the burn-in period of order $\log{T}$, $c$ is the union upper bound on the cost $c_t$, $l$ is the length of the signature of the participants, and $k_i$ is the threshold parameter used in the construction of $A_t$. Here the set $A$ is defined as $A= \{\forall 1 \leq t  \leq T, b_t =1\}$ which satisfies that $P(A) \geq 1- \frac{1}{l^{T-1}}$. 
\end{theorem}

\begin{proof}[Proof Sketch]
   The complete proof is deferred to Appendix; we present the main logic here. The new estimators are consistent with the ones in \citep{vial2021robust}, and thus $\hat{\mu}_i$ is close to the ones from honest participants, as well as $\Tilde{\mu}_i$. Moreover, the contraction by $1-P_t$ ensures the convergence of $\Tilde{\mu}_i$ to the ground truth $\mu_i$, with the rate aligning with the order in the definition of $Dist(\Tilde{\mu}_i(t), \mu_i)$ and the exploration length in the decision rule (UCB). Therefore, we establish the upper bound on the cumulative cost and the upper bound on the reward difference between the pulled arm and the optimal arm. 
   
\end{proof}

\begin{remark}
We would like to emphasize that there should be at least one honest commander who has the same message as the honest validators. The honest validators choose to do majority voting only when the received message matches their own. In other words, consensus alone is not sufficient for the protocol; rather, consensus on the correct estimators guarantees the desired functionality of the protocol.
\end{remark}

\subsection{$M_A \leq \frac{3}{4}M$ with $M$ commanders and distance-based cost}

While the $\frac{1}{2}$ assumption is consistent with the existing literature in robust optimization, there is a line of work that does not rely on this assumption, though it is in a homogeneous setting. Surprisingly, we report that by more precisely characterizing the different types of malicious behaviors, we can relax the assumption on the number of malicious participants.

First, we introduce the structure of the malicious participants' behaviors.  
\paragraph{Structure of malicious behaviors} We define set $M^1_A \subset M_A$ as comprising of malicious participants that only perform attacks on the estimators. In the meantime, we denote  $M^2_A \subset M_A$ as the set comprising of malicious participants that perform attacks on the consensus. Moreover, $M^{2,1}_A \subset M^{2}_A$ are the malicious participants that perform attacks on both the estimators and the consensus. Note that all the malicious participants are allowed to perform existential forgery on the signatures of the honest participants. 

Next, we deploy a different update rule for the honest participants as in Option 3. 

\begin{remark}
    The global estimator $\Tilde{\mu}_i(t)$ in Global Update is constructed as 
    \begin{align*}
     \Tilde{\mu}_i(t) = P_t\Tilde{\mu}_i(t-1) + (1-P_t)\hat{\mu}_i(\tau)
    \end{align*}
\end{remark}
where $P_t = 1 - \frac{1}{t}$ and again $\tau = \max_{s < t}\{b_s = 1\}$.

\begin{remark}
     Note that the construction of set $D_t$ is feasible, as the honest participant can track the public key (the signature) through tracing back a Chandelier tree, and thus track the label through the fixed mapping between the participants' public keys and the labels.
\end{remark}

With this, the formal statement reads as follows. 

\begin{theorem}
Let us assume that the total number of honest participants is at least $\frac{1}{4}M$. Let us assume that $M^1_A < M_H - 1$ and $M^2_A < \frac{1}{2}M - 1$. The cost is the aforementioned distance-based cost. Meanwhile, consistent with the standard assumptions, let us assume that the malicious participants perform existential forgery on the signatures of honest participants with an adaptive chosen message attack. Lastly, let us assume that the participants are in a standard universal composability framework when constructing $A$ where the set $A$ is defined as $A= \{\forall 1 \leq t  \leq T, b_t =1\}$ which satisfies that $P(A) \geq 1- \frac{1}{l^{T-1}}$. Then we have that 
\begin{align*}
    E[R_T|A] \leq 3\cdot L + O(\log{T}) + |M_H|Kl^{1-T}
\end{align*}
where 
$L$ is the length of the burn-in period, $l$ is the length of the signature of the participants, and $k_i$ is the threshold parameter used in the construction of $A_t$ defined in Option 3.  
\end{theorem}

\begin{proof}[Proof Sketch]
    The full proof is presented in Appendix; the main logic is shown as follows.   For participants that perform attack on the consensus protocol, i.e. $j \in M_A^2$, their local estimators are not taken into consideration, since the honest participants are able to identify the pubic keys associated with the consensus message on a Chandeller tree, and thus the label of the participants. As a result, the estimator candidates are either from $M_A^1$ or $M_H$. Note that $M_A^1 < M_H-1$, which immediately implies that the     \begin{align*}
        \hat{\mu}_i(t) = \frac{\sum_{m \in B_t}\bar{\mu}_i^m(t)}{|B_t|}
    \end{align*}
    is close enough to the ground truth, from where the rest of the analysis follows from Theorem 2.  
    
\end{proof}

\subsection{Any $M_A$ with $M$ commanders and distance-based cost}

Besides the $\frac{3}{4}$ assumption, more surprisingly, we find that this brand new algorithmic framework works for more general settings with any number of participants, assuming a more refined structure of the malicious participants. More specifically, if the malicious participants that attack the consensus step only perform this type of attack without attacking the estimators, then the number of malicious participants can be any value larger than 1. The cost definition is again the distance-based one. The formal statement reads as follows.
\begin{theorem}
Let us assume that the total number of honest participants is any. Let us assume that $M^1_A < \frac{1}{2}M - 1$ and $M^2_A < \frac{1}{2}M - 1$, and further assume that the participants in $M_A^2$ only perform attacks to the consensus. The cost is the distance-based cost. Meanwhile, let us assume that the malicious participants perform existential forgery on the signatures of honest participants with an adaptive chosen message attack. Lastly, let us assume that the participants are in a standard universal composability framework when constructing $A$. Then we have that 
\begin{align*}
    E[R_T|A] \leq 3\cdot L + O(\log{T}) + |M_H|Kl^{1-T}
\end{align*}
where 
$L$ is the length of the burn-in period, $l$ is the length of the signature of the participants, and $k_i$ is the threshold parameter used in the construction of $A_t$. Here the set $A$ is defined as $A= \{\forall 1 \leq t  \leq T, b_t =1\}$ which satisfies that $P(A) \geq 1- \frac{1}{l^{T-1}}$.
\end{theorem}

\begin{proof}[Proof Sketch]
   The complete proof is deferred to Appendix; the proof logic is as follows. Under the assumption on the participant structure, the estimators computed for the validation $\Tilde{\mu}$ are the same as in Theorem 2, as are the theoretical properties, which first guarantee the convergence of the algorithms and second control the associated cost. Meanwhile, since the number of participants performing consensus attacks is less than $\frac{1}{2}M$, the consensus runs successfully. Subsequently, the analysis of Theorem 2 also holds, which concludes the proof. 
   
\end{proof}

\subsection{Any $M_H$ with an efficient commander selection protocol}

So far, what we have discussed assumes a fixed number of commanders, i.e., $M$, during the consensus step. In other words, every validator serves as a commander once in a single run of the consensus protocol. While this guarantees the decentralization of the coordination mechanism, there is room for improvement in efficiency. As an extension, we consider a more general commander selection procedure in the protocol, with adaptive numbers of commanders depending on the history, to improve efficiency while ensuring decentralization. 

\paragraph{Commander selection}

The commander set $C^s_t$ is determined by executing Algorithm 3 where the trust coefficients $w_{m}(t)$ are the probabilities of being selected as commanders. Formally, for any participant $m$, the probability $P(m \in C^s_t) = w_m(t)$.  

It is worth noting that these probabilities are independent of one another, which results in geometric distributions. Therefore, depending on whether the participants are malicious or honest, we specify $w_m(t)$ as time-invariant values. The exact values and the intuition behind the choices are as follows. Specifically, $w_m(t) = w_{m} = 1 - \frac{\log{T}}{T}$, for any $m \in M_H$. Consider the event of whether honest participant $m$ is selected as a commander as $E_{m}^t$. In other words, $E_m^t =1$ if participant $m$ is a commander and 0 otherwise. Define $E_t$ as $\cap_{m \in M_H}\{E_{m}^t = 0\}$. Then we have that 
    \begin{align*}
        E[\sum_{t=1}^TE_t] & = \sum_{t=1}^TE[\cap_{m \in M_H}\{E_{m}^t = 0\}] \\
        & \leq \sum_{t=1}^T\sum_{m \in M_H}E[\{E_{m}^t = 0\}] \\
        & = \sum_{t=1}^T\sum_{m \in M_H}(1 - w_m(t)) = \log{T}.
    \end{align*}
    It implies that for at most $\log{T}$ steps, there is no honest commander, and thus the consensus fails, which is neglectable compared to the time horizon $T$. Differently, we let $w_j(t) = w_{j} = \frac{\log{\frac{|M_A|}{\eta}}}{L}$, for any $j \in M_A$. Then we consider the event of whether malicious participant $j$ is selected as a commander or not, namely, $F^j_t$. Likewise, $F^j_t = 1$ if participant $j$ is a commander and $0$ otherwise. Define $F_t = \cap_{j \in M_A}\{\exists s \leq t, s.t. F^j_s = 1\}$. Then we obtain 
    \begin{align*}
      P(F_t) & = P(\cap_{j \in M_A}\{\exists s \leq t, s.t. F^j_s = 1\}) \\
      &  \geq 1 - \sum_{j \in M_A}P(\{\forall s \leq t, s.t. F^j_s = 0\}) \\
      & = 1 - \sum_{j \in M_A}(1 - w_j)^t \\
      & = 1 - |M_A|(1 - w_j)^t \\
      & \geq 1 - |M_A|e^{-w_jt}
    \end{align*}
    By the choice of $w_j = \frac{\log{\frac{|M_A|}{\eta}}}{L}$, we derive that $P(F_{L}) \geq 1 - |M_A|e^{-w_jt} = 1 - \eta$. That is to say, with high probability, all the malicious validators are identifiable, and then the corresponding estimators from them will be excluded. 

\paragraph{Consensus Protocol}

For each commander $m \in C^s_t$, we run the consensus algorithm with this commander  as in \citep{lamport2019byzantine}.

With this new protocol, the honest participants construct a new filter list based on Option 3. That is to say that any malicious validators performing attacks on the consensus are added to the block list. 

When updating estimators $\hat{\mu}_i(t)$, the honest participants use the same formula as 
\begin{align*}
    \hat{\mu}_i(t) = \frac{\sum_{j \in B_t}\bar{\mu}^j_i(t)}{|B_t|}
\end{align*}

Subsequently, we establish the following regret bound with the updated protocol and associated estimators. 
\begin{theorem}
Assume the same conditions as in Theorem 4. Let us assume that the commanders are selected based on the above protocol, and the estimators are computed aforementioned. Then we obtain that the regret upper bound with respect to our algorithm is $O(\log{T})$. 
\end{theorem}

\begin{proof}[Proof Sketch]
The complete proof is in Appendix; we present the main logic here. Based on the commander selection protocol, we establish that after the burn-in period, there is at least one honest commander and the information from that commander is correct. Since the total number of honest validators is at least $\frac{1}{2}M_H + 1$ So the validators achieves consensus regarding the correct information. As a result, all the remaining analysis follows from Theorem 4. 

\end{proof}

\subsection{Any $M_H$ with constant cost}

Recall that with the assumption of at most $\frac{1}{3}$ participants are malicious, we have established the regret bound when the cost is constant. Without this assumption, we have proved the regret assuming distance-based cost, which highlights a gap. To this end, we next consider the constant cost that imposes more penalization and show the corresponding result. Intuitively, if the information from malicious participants is close enough to that from honest participants, the cost will always be constant, and thus the regret will be linear in $T$. As a result, we propose the following definition characterizing the difference between the two groups of participants and introduce the assumption accordingly. 

\paragraph{Pre-fixed $\epsilon$-safe zone}

A pre-fixed $\epsilon, \delta$-safe zone is defined as a set of participants $S_{\epsilon}$, such that for any participant $j \in S_{\epsilon}$ and any arm $1 \leq i \leq K$, we have that the   
\begin{center}
   $f^j_i = (1-\epsilon) \cdot h^m_i + \epsilon \cdot q_i^m$
\end{center}
where $f^j_i$ is the corresponding black-box reward generator for, $h^m_i$ is the corresponding stochastic reward generator for arm $i$ at honest participant $m$ and $q_i^m$ follows a unknown but fixed distribution different from that of $h^m_i$.

\begin{remark}
This assumption separates the malicious participants from the honest participants to make the malicious participants distinguishable, thereby eliminating the estimators from malicious participants. It is worth noting that this assumption is consistent with the existing literature \citep{dubey2020private}, which adopts the same principle when considering malicious behavior.
\end{remark}

Moreover, this assumption can be relaxed to the following version where the minimum gap, instead of the exact gap, is $\epsilon$, which measures the difference between the estimators from the malicious participants and those from the honest participants.

\paragraph{Pre-fixed $\epsilon$-safe zone}

A pre-fixed $\epsilon$-safe zone is defined as a set of participants $S_{\epsilon}$, such that for any participant $m \in S_{\epsilon}$ and any arm $1 \leq i \leq K$, we have that 
\begin{center}
   $f^m_i- h^m_i \geq  \epsilon \cdot q_i^m$. 
\end{center}
where $f^m_i$ is the corresponding black-box reward generator, $h^m_i$ is the corresponding stochastic reward generator for arm $i$ at honest participants and $q_i^m$ follows a unknown but fixed distribution with mean value $1$.

Subsequently, we introduce the assumption on the participant structure as follows.

\begin{assumption}{(Pre-fixed)}
The pre-fixed $\epsilon$-safe zone contains no malicious participants that only perform attacks on the estimators, namely, $M_A^1 \cap S_{\epsilon}  = \emptyset$. 
\end{assumption}

We update the estimator computation, where the global estimator $\Tilde{\mu}_i(t)$ is constructed as 
    \begin{align*}
     \Tilde{\mu}_i(t) = P_t\Tilde{\mu}_i(t-1) + (1-P_t)\hat{\hat{\mu}}_i(t-1)
    \end{align*}
and 
$\hat{\hat{\mu}}_i(t)$ is constructed as     \begin{align*}
        \hat{\hat{\mu}}_i(t) = \frac{\sum_{m \in C_t}\bar{\mu}_i^m(t)}{|C_t|}
    \end{align*} with $C_t = \{1 \leq j \leq M: |\hat{\mu}_i(t) - \bar{\mu}^j_i(t)| \leq \frac{\epsilon}{2}\}$
where $\hat{\mu}_i(t)$ is defined as 
    \begin{align*}
        \hat{\mu}_i(t) = \frac{\sum_{m \in B_t}\bar{\mu}_i^m(t)}{|B_t|}
    \end{align*}
where $B_t$ is defined as before.

Next, we demonstrate the corresponding regret bound under this assumption and constant cost.

\begin{theorem}
    Assume the same conditions as in Theorem 4. Let us further assume that Assumption 1 holds. With the new rule of updating the estimators, the regret bound of the proposed algorithm is $O(\log{T})$. 
\end{theorem}

\begin{proof}[Proof Sketch]
    The detailed proof is presented in Appendix; the idea is as follows. With the pre-fixed $\epsilon$-safe zone, the difference between $\bar{\mu}_i^m(t)$ of malicious participant $m$ and $\bar{\mu}_i^j(t)$ of honest participant $j$ is at least $\epsilon$. In the meantime, after the burn-in period, $\hat{\mu}_i(t)$ has at most $\frac{1}{4}\epsilon$ difference compared to $\mu_i$, which is also true for the distance between $\bar{\mu}^m_i(t)$ and $\mu_i$. This leads to $|\hat{\mu}_i(t) - \bar{\mu}^j_i(t)| \leq \frac{\epsilon}{2}$, and thus $|\hat{\mu}_i(t) - \bar{\mu}^m_i(t)| > \frac{\epsilon}{2}$. As a result, $C_t$ only contains honest participants, which implies that the cost is $0$. The remaining analysis on the regret based on the decomposition step is the same as that of Theorem 4, which completes the proof. 
    
\end{proof}


\subsection{Any $M_H$ with an efficient validator selection protocol} 

What we have discussed presumes that the validator set includes the entire participant set. While this choice guarantees full decentralization of the consensus protocol, as every participant has a say in the protocol, it may lack efficiency, taking significant time to achieve consensus. This is supported by the well-known dilemma of balancing decentralization, which requires many participants to be validators, and efficiency, which tends to reduce the number of validators. To address this, we propose a new approach to select validators based on a newly defined reputation score system motivated by the Proof-of-Authority concept \citep{fahim2023blockchain}, allowing the flexibility of selecting any number of validators ranging from  $M_H$ to $2M_H-1$.   

More specifically, our reputation score system runs as follows. For each participant $i$, its reputation at time step $t$ is computed by a smart contract as follows. Formally, it reads as $RS_i^t = G(U_i^t)$, where $G$ is any monotonicity preserving function such that if $x_1 \leq x_2$, then $G(x_1) \leq G(x_2)$, and $U_i^t$ quantifies the accuracy of the information from participant $i$ at time step $t$, which is defined as
\begin{align*}
    U_i^t = \sum_{j=1}^{K}-(\bar{\mu}_j^i(t) - \Tilde{\mu}_j(t))^2 - \epsilon^2(\overset{\Delta}{\mu}_j^i(t)- \Tilde{\mu}_j(t))^2)^2 
\end{align*}
where $\overset{\Delta}{\mu}_j^i(t)$ denotes the estimator for arm $j$ given by participant $i$ after the consensus step, and $\bar{\mu}_j^i(t), \Tilde{\mu}_j(t)$ are the aforementioned estimators for arm $j$.

To proceed, our validator selection procedure is as follows. 

\paragraph{Validator Selection} First we rank the reputations of the participants and record the participants as $\{l_1, l_2, \ldots, l_M\}$ accordingly, where $l_1$ represents the label of participant with the largest reputation. Then given the preference of the protocol, we select the top $N$ participants, namely, $\{l_1, l_2, \ldots, l_N\}$, where $N$ is any number between $M_H$ and $2M_H-1$.

It is worth mentioning that the reputation system also ensures the fairness of the protocol, or equivalently, decentralization, as no single participant is favored and the criterion is merit-based, depending on how much they contribute to the protocol. Also, privacy is maintained since the participants are not aware of $U_i^t$ due to the existence of $G(\cdot)$. Meanwhile, the number of validators given by the reputation score system is flexible in the range of $[M_H, 2M_H-1]$, balancing the trade-off between decentralization and efficiency. While it is practically meaningful, it is also crucial to demonstrate the theoretical effectiveness of the coordination mechanism after incorporating the reputation score system. Subsequently, we present the following theoretical regret guarantee of the entire system with the above reputation score system for validator selection. The formal statement reads as follows. 
\begin{theorem}
 Assume the same conditions as in Theorem 4. Let us further assume that Assumption 1 holds and that the validators are selected based on Validator Selection. Let us assume that the cost is the constant. 
 Then we have that 
\begin{align*}
    E[R_T|A] \leq 3\cdot L + O(\log{T}) + |M_H|Kl^{1-T}
\end{align*}
where 
$L$ is the length of the burn-in period, $l$ is the length of the signature of the participants, and $k_i$ is the threshold parameter used in the construction of $A_t$. Here the set $A$ is defined as $A= \{\forall 1 \leq t  \leq T, b_t =1\}$ which satisfies that $P(A) \geq 1- \frac{1}{l^{T-1}}$. 
\end{theorem}

\begin{proof}[Proof Sketch]
The complete proof is deferred to Appendix; the main logic is as follows. We show that the reputations of honest participants are always higher than those of the malicious participants after the burn-in period. This is true, noting that the reputation of a malicious participant attacking the consensus will significantly decrease, while the reputation of a malicious participant attacking the estimators is also lower than that of an honest participant, since it has a larger first term in $U_i^t$ and due to the monotonicity property of $G(\cdot)$. Therefore, by selecting the validators, the proportion of honest participants in the validator set is at least $\frac{1}{2}$, which ensures that the consensus will be achieved. The rest of the regret analysis follows that of Theorem 6.   

\end{proof}

\begin{remark}
    It is worth noting that  existing work, such as \citep{dennis2015rep, zhou2021blockchain, arshad2022reputable}, has proposed reputation-based validator selection. However, most of this work focuses on the practical performance of the reputation system, with little theoretical analysis on the security guarantee. Here, we prove that the reputation system ensures optimal regret, which is only obtainable when the coordination mechanism is secure enough in terms of the consensus and the associated information after the consensus.
\end{remark}

\section{Other Performance Measure}

While we have established various theoretical bounds on the regret of the coordination mechanism, demonstrating the algorithm's optimality, it is worth noting that security has been a crucial aspect of building fault-tolerant systems. In fact, we ensure that the security guarantee is necessary for the coordination mechanism's optimality, which is connected through our proposed framework as part of our contributions. In other words, security is an implication of the shown regret bounds. More specifically, we consider the following factors that affect security and illustrate how they are connected with regret. In the meantime, recall that to incentivize the participation of participants, we invented a new cost mechanism, motivated by \citep{murhekar2024incentives}. While our setting is not completely zero-sum, which does not enable the full characterization of Nash Equilibrium, the two different groups of participants, namely, malicious participants and honest participants, have conflicting objectives. To this end, we provide a qualitative discussion by illustrating the trade-offs faced by malicious participants and point out potential future directions regarding the cost mechanism.
\subsection{Security of the Protocol}
\paragraph{Digital signature 
} The security of the coordination mechanism partially depends on the reliability of the signature scheme, as it determines whether the participant can maintain its own signature and the corresponding mapping between the label and signature. Note that the employment of the digital signature scheme \citep{goldwasser1988digital} is in a plug-in fashion, independent of everything else. As a result, the theoretical guarantee still holds, implying the security of the coordination mechanism and serving as a prerequisite (assumption) needed for achieving consensus when running the Byzantine Fault Tolerant protocol. 

\paragraph{Consensus} The security of the consensus protocol also plays an important role in the coordination mechanism's security, as no single participant can determine the estimator to be sent to the smart contract. This prevents malicious participants from manipulating the estimators but adds additional challenges for honest participants. By deploying the Byzantine Fault Tolerant protocol with the digital signature scheme and our newly proposed commander selection procedures, we guarantee that both consensus and good enough estimators are achieved with high probability. Only in this case can the regret be optimized, which implies that optimal regret indicates the security of the consensus protocol.

\paragraph{Privacy} Another main aspect of security is whether the participants' information is accessible to others, namely the degree of privacy preservation. We note that though the empirical reward estimators are available, the number of arm pulls is not broadcast. This prevents malicious participants from retrieving the reward and arm sequence, thus protecting privacy. Moreover, the rule for computing reputation is unknown to the participants, as it is implemented through a smart contract, which prevents malicious participants from manipulating the reputation. The correctness of the reputation is essential to the consensus protocol and thus the regret. In other words, the optimality of the regret also implies the correct execution of the reputation system.

\subsection{Optimality of the Cost Mechanism}
This cost mechanism is consistent with the one in \citep{murhekar2024incentives}, by adding a cost term to the original reward. While their cost depends on how many samples a participant contributes, we measure how much contribution a participant makes to the validated estimators. Honest participants need to identify the malicious participants and gain knowledge about the reward to maximize their reward function.

Assuming the cost is constant, the optimal strategy for malicious participants is to send sufficiently accurate information so that the honest participants cannot determine their identities, which implies that there is no Nash Equilibrium. Next, we provide some insights into whether an equilibrium exists between malicious and honest participants when adopting the distance-based cost. If the malicious participants keep broadcasting incorrect estimators, they will be excluded from consideration by the honest participants, allowing honest participants to incur a smaller cost. On the other hand, if the malicious participants send accurate enough information, the cost for honest participants is small as well, by definition. This implies that our proposed mechanism captures the trade-off and has the potential to uncover the Nash equilibrium with respect to how malicious participants transmit their estimators. We point out that quantitatively and rigorously characterizing the equilibrium presents a very promising direction, which goes beyond the scope of this paper.

\section{Conclusion and Future Work}

This paper considers the robust multi-agent multi-armed bandit (MAB) problem within the framework of system security, representing the first work to explore online sequential decision-making with participants distributed on a blockchain. The introduction of conditionally observable rewards and the penalization of inaccurate information brings new challenges, while taking security and privacy into consideration, besides optimality, distinct from blockchain-based federated learning or Byzantine-resilient multi-agent MAB. To solve the problem, we propose a new methodological approach combining the strategy based on Upper Confidence Bound (UCB) with blockchain techniques and invent new modifications. On blockchain, a subset of participants forms a validator set responsible for information integration and achieving consensus on information transmitted by all participants. Consensus information is then sent to a smart contract for verification, with approved blocks only upon successful verification. The environment determines and sends the reward information to the participants based on the interaction with the smart contract. As part of our contributions, we use reputation to determine the validator selection procedure, which depends on the participants' historical behaviors. Additionally, we incorporate a digital signature scheme into the consensus process, eliminating the traditional $\frac{1}{3}$ assumption of the Byzantine general problem. Furthermore, we introduce a cost mechanism to incentivize malicious participants by rewarding their contributions to the verification step. We provide a comprehensive regret analysis demonstrating the optimality of our proposed algorithm under specific assumptions, marking a breakthrough in blockchain-related learning tasks, which has seen little analysis. To conclude, we also include a detailed discussion on the security and privacy guarantees.

Moving forward, we suggest potential future directions as the next step. While our framework works for a general number of malicious participants, it relies on assumptions about the structure of malicious behaviors. Removing such assumptions could generalize the problem setting. Meanwhile, we consider two types of attacks related to the framework—those targeting the estimators and those targeting the consensus—and note that there is a rich body of literature on different aspects of security attacks. Incorporating these into the framework is both meaningful and promising. Lastly, we emphasize that mechanism design has great potential in online learning, especially in a multi-agent system, to ensure that participants perform as expected. We hope that this paper can pave the way for combining the rich literature in mechanism design with multi-agent learning systems, in the era of cybersecurity and mixed-motive cooperation.

\newpage

\bibliography{neurips_2023}

\newpage

\appendix

\section*{Proof of Results in Section 4}

\subsection*{Proof of Theorem 1}
\begin{proof}
For regret, we have the following decomposition. 
Let us denote $b_t$  as the indicator function of whether the block at time step $t$ is approved. Likewise, for any time step $t$, we denote whether the estimators from the malicious participants are utilized in the integrated estimators as $h_t$. Let the length of the burn-in period be $L$. 

Note that 
\begin{align*}
    R_T 
    & = \max_i\sum_{m \in M_H}\sum_{t=1}^T\mu_i - \sum_{m \in M_H}\sum_{t=1}^T
(\mu_{a_m^t}^b - c_t) \\
    & = \max_i\sum_{m \in M_H}\sum_{t=1}^T\mu_i - \sum_{m \in M_H}\sum_{t=1}^T\mu_{a_m^t}^b + \sum_{m \in M_H}\sum_{t=1}^Tc_t \\
    & = \max_i\sum_{m \in M_H}\sum_{t=1}^T\mu_i - \sum_{m \in M_H}\sum_{t=1}^T\mu_{a_m^t}1_{b_t =1} + \sum_{m \in M_H}\sum_{t=1}^Tc_t \\
    & = \max_i\sum_{m \in M_H}\sum_{t=1}^T\mu_i - \sum_{m \in M_H}\sum_{t=1}^T\mu_{a_m^t}1_{b_t =1} + \sum_{m \in M_H}\sum_{t=1}^Tc1_{h_t = 1} \\
\end{align*}

Meanwhile, the regret can be bounded as follows
\begin{align}\label{eq:R_T}
    R_T & \leq  L + c\cdot L + \sum_{t=L+1}^T\sum_{m \in M_H}(\mu_{i^*} - \mu_{a_m^t}1_{b_t =1}) + \sum_{m \in M_H}\sum_{t=L+1}^Tc1_{h_t = 1} \notag \\
   &  \doteq (c+1)\cdot L + T_1 + T_2
\end{align}

We start with the second term $T_2$. Note that $h_t =1$ is equivalent to $\{m: m \in B_t \cap m \not\in M_H\} \neq \emptyset$. Note that because the cost is positive, $B_t$ is nonempty.  

By taking the expectation over $T_2$, we derive
\begin{align*}
    E[T_2|A] & = \sum_{m \in M_H}\sum_{t=L+1}^TcE[1_{h_t = 1}] \\
    & = \sum_{m \in M_H}\sum_{t=L+1}^TcE[1_{\{m: m \in B_t \cap m \not\in M_H\} \neq \emptyset}] 
\end{align*}

Based on Lemma 2 in \citep{zhu2023byzantine}, we obtain that 
\begin{align*}
   1_{\{m: m \in B_t \cap m \not\in M_H\} \neq \emptyset} = 1_{|A_t| < 2f} 
\end{align*}
which immediately implies that 
\begin{align*}
    E[T_2|A] & = \sum_{m \in M_H}\sum_{t=L+1}^TcE[1_{h_t = 1}] \\
    & = \sum_{m \in M_H}\sum_{t=L+1}^TcE[1_{\{m: m \in B_t \cap m \not\in M_H\} \neq \emptyset}] \\
    & = \sum_{m \in M_H}\sum_{t=L+1}^TcE[1_{|A_t| < 2f}]. 
\end{align*}

In the meantime, we note that for any honest validators, the choice of $A_t$ guarantees that honest participants are included after the burn-in period. More specifically, the set of $A_t$ satisfies that for any validator $j \in M_H$, 
\begin{align*}
    m \in A_t \Leftrightarrow k_in_{m,i}(t) > n_{j,i}(t) \Leftrightarrow m \in M_H
\end{align*} 
where $1 < k_i < 2$. This condition holds at the end of burn-in period which is straightforward since each honest. After the burn-in period, the honest participants has the same decision rule
\begin{align*}
    a_m^t = argmax_i \Tilde{\mu}_{i}^m(t) + F(m,i,t)
\end{align*}
where $\Tilde{\mu}_{i}^m(t) = \tilde{\mu}_{i}^b(t)$. In other words, each honest participant uses the validated estimator $\tilde{\mu}_{i}^b(t)$.Since both  $n_{m,i}(t)$ and $n_{j,i}(t)$ are larger than $\frac{L}{K}$, then we have that there exists $k_i = \frac{n_{j,i}(t)K}{L}$, such as  $k_in_{m,i}(t) > n_{j,i}(t)$ for every $m \in M_H$.

This implies that 
\begin{align}\label{eq:A_t-f}
   A_t > |M_H| \geq 2f 
\end{align}
by the assumption that the number of honest participants is at least $\frac{2}{3}M$. 

That is to say, 
\begin{align*}
    E[1_{|A_t| > 2f}] = 1
\end{align*}
and subsequently, we have 
\begin{align*}
    E[T_2|A] & = \sum_{m \in M_H}\sum_{t=L+1}^TcE[1_{|A_t| < 2f}] \\
    & = 0 
\end{align*}

We note that the construction of $A_t$ is done without knowing the number of pulls of arms of other participants. This is realized by using the homomorphic results, Theorem 5.2 as in \citep{asharov2012multiparty} under the universal composability framework.  

Next, we proceed to bound the first term $T_1$. Note that 
\begin{align*}
E[T_1|A] & \leq \sum_{t=L+1}^T\sum_{m \in M_H}(\mu_{i^*} - \mu_{a_m^t}1_{b_t =1}) \\
& = (T-L) \cdot |M_H|\cdot \mu_{i^*} - \sum_{m \in M_H}\sum_{t=L}^TE[\mu_{a_m^t}|b_t = 1]P(b_t = 1)  
\end{align*}

In the meantime, we obtain the following 
\begin{align*}
    & E[\mu_{a_m^t}|b_t = 1] \\
    & = E[\sum_{k=1}^K\mu_k\cdot 1_{a_m^t =k} | b_t = 1] \\
    & = \sum_{k=1}^KE[\mu_k1_{a_m^t = k}|b_t = 1] \\
    & \geq \sum_{k=1}^K\mu_k \cdot \frac{1}{P(b_t =1)} \cdot (E[1_{a_m^t = k}] - P(b_t = 0)).
\end{align*}

This immediately gives us that 
\begin{align}\label{eq:E_T_2}
    & E[T_1|A] \notag \\
    & \leq (T-L)|M_H|\mu_{i^*} - \sum_{m \in M_H}\sum_{t=L}^T(\sum_{k=1}^K\mu_k \cdot \frac{1}{P(b_t =1)} \cdot (E[1_{a_m^t = k}] - P(b_t = 0))) P(b_t = 1) \notag \\
    & = (T-L)|M_H|\mu_{i^*} - \sum_{m \in M_H}\sum_{t=L}^T(\sum_{k=1}^K\mu_k(E[1_{a_m^t = k}] - P(b_t = 0)) \notag \\
    & = (T-L)|M_H|\mu_{i^*} -\sum_{m \in M_H}\sum_{t=L}^T\sum_{k=1}^K \mu_kE[1_{a_m^t = k}] + \sum_{m \in M_H}\sum_{t=L}^T\sum_{k=1}^K \mu_kP(b_t = 0). 
\end{align}

Based on Theorem 2 in \citep{lamport2019byzantine}, the consensus is achieved, i.e. $b_t = 1$, as long as the digital signatures of the honest participants can not be forged. Based on our assumption, we have that the malicious participants can only perform existential forgery on the signatures of the honest participants and the attacks are adaptive chosen-message attack. Then based on the result, Main Theorem  in \citep{goldwasser1988digital}, the attack holds with probability at most $\frac{1}{Q(l)}$ for any polynomial function $Q$ and large enough $l$ where $l$ is the length of the signature. 

More precisely, we have that with probability at least $1- \frac{1}{Tl^{T-1}}$, the signature of the honest participants can not be forged, and thus, the consensus can be achieved, i.e.
\begin{align}\label{eq:b_t}
   P(b_t =1) \geq 1- \frac{1}{Tl^{T-1}}.  
\end{align}

Subsequently, we derive that 
\begin{align*}
   (\ref{eq:E_T_2}) & \leq (T-L)|M_H|\mu_{i^*} -\sum_{m \in M_H}\sum_{t=L}^T\sum_{k=1}^K \mu_kE[1_{a_m^t = k}] + \sum_{m \in M_H}\sum_{t=L}^T\sum_{k=1}^K \mu_k(\frac{1}{Tl^{T-1}}) \\
   & \leq \sum_{m \in M_H}\sum_{t=L}^T(\mu_{i^*} - \sum_{k=1}^K \mu_kE[1_{a_m^t = k}]) + |M_H|Kl^{T-1} \\
   & = \sum_{m \in M_H}\sum_{k=1}^K\Delta_kE[n_{m,k}(t)] + |M_H|Kl^{T-1} \\
   & \doteq T_{21} + |M_H|Kl^{T-1}
\end{align*}

And for each honest participant, they are using the estimators based on the validated estimators, as long as the block is approved. Consider the following event, $A = \{\forall 1 \leq t  \leq T, b_t =1\}$. Based on (\ref{eq:b_t}) and the Bonferroni's inequality, we obtained that 
\begin{align*}
    P(A) & = P(\forall 1 \leq t  \leq T, b_t =1) \\ 
    & = 1 - P(\exists 1 \leq t \leq T, b_t = 0) \\
    & \geq 1 - \sum_{t=1}^TP(b_t=0)\\
    & \geq 1 - \frac{1}{l^{T-1}}. 
\end{align*}

On event $A$, the blockchain always gets approved, and then all the honest participants follow the validated estimators from the validators. By (\ref{eq:A_t-f}) and Lemma 2 in \citep{zhu2023byzantine}, we have that the validated estimator $\Tilde{\mu}_{i}(t)$ can be expressed as 
\begin{align*}
    \hat{\mu}_{i}(t) = \sum_{j \in A_t \cap M_H}w_{j,i}(t)\bar{\mu}_{i}^j(t)
\end{align*}
where the weight $w_{j,i}(t)$ meets the condition 
\begin{align*}
    \sum_{j \in A_t \cap M_H}w_{j,i}(t) = 1, 
\end{align*}
which immediately implies that 
\begin{align*}
    E[\hat{\mu}_{i}(t)] = \mu_i.
\end{align*}

We note that the variance of $\hat{\mu}_{i}(t)$, $var(\hat{\mu}_{i}(t))$, satisfies that, 
\begin{align*}
    var(\hat{\mu}_{i}(t)) & = var(\sum_{j \in A_t \cap M_H}w_{j,i}(t)\bar{\mu}_{i}^j(t)) \\
    & \leq |A_t \cap M_H|\sum_{j \in A_t \cap M_H}w_{j,i}(t)^2var(\bar{\mu}_{i}^j(t))) \\
    & \leq |A_t \cap M_H|\sum_{j \in A_t \cap M_H}w_{j,i}^2(t)\sigma^2\frac{1}{n_{j,i}(t)} \\
    & \leq |A_t \cap M_H|\sum_{j \in A_t \cap M_H}w_{j,i}^2(t)\sigma^2\frac{k_i}{n_{m,i}(t)} \\
    & = |M_H|\frac{k_i}{n_{m,i}(t)}\sum_{j \in M_H}w_{j,i}^2(t)\sigma^2 \\
    & \leq |M_H|\frac{k_i\sigma^2}{n_{m,i}(t)}
\end{align*}
where the inequality holds by the Cauchy-Schwarz inequality, the second inequality holds by the definition of sub-Gaussian distributions, the third inequality results from the construction of $A_t$, and the last inequality is as a result of $(a+b)^2 \geq a^2+b^2$. 

Next, we show by induction that $var(\Tilde{\mu}_i(t)) \leq 3|M_H|\frac{k_i\sigma^2}{n_{m,i}(t)}$ for $t \geq 3K$.

At time step $3K$, we have that $var(\Tilde{\mu}_i(t)) \leq 1$ since $E[\Tilde{\mu}_i(t)] = \mu_i \leq 1$. In the meantime, 
\begin{align*}
   & 3|M_H|\frac{k_i\sigma^2}{n_{m,i}(t-1)} \\
   & \geq  3|M_H|\frac{k_i\sigma^2}{3} \\
   & = |M_H|k_i\sigma^2 \geq 1
\end{align*}
since we have $k_i \geq 1$ and $\sigma^2 \geq \frac{1}{M_H}$.

First, assume that for $t-1$, we have $var(\Tilde{\mu}_i(t-1)) \leq 3|M_H|\frac{k_i\sigma^2}{n_{m,i}(t-1)}$. 

Meanwhile, by the update rule such that $\Tilde{\mu}_i(t) = (1-P_t)\hat{\mu}_i(t) + P_t\Tilde{\mu}_i(\tau)$ where $\tau = \max_{s< t}\{b_s = 1\}$. 

Note that with probability at least $P(A) = 1 - \frac{1}{l^{T-1}}$, $b_s=1$ for all $s < t$. This implies that on event $A$, $\tau = t-1$. Therefore, by the cauchy-schwartz inequality, we obtain that  
\begin{align*}
   var(\Tilde{\mu}_i(t)) & \leq 2(1-P_t)^2(var(\hat{\mu}_i(t))) +  2P_t^2var(\Tilde{\mu}_i(t-1)) \\
   & \leq \frac{1}{2}|M_H|\frac{k_i\sigma^2}{n_{m,i}(t)} + \frac{1}{2}3|M_H|\frac{k_i\sigma^2}{n_{m,i}(t-1)} \\
   & \leq 3|M_H|\frac{k_i\sigma^2}{n_{m,i}(t)}
\end{align*}
where the last inequality holds by the fact that $n_{m,i}(t-1) \geq  n_{m,i}(t) -1 \geq \frac{2}{3}n_{m,i}(t)$ when $t > 3 \cdot K$.

Subsequently, we have that 
\begin{align}\label{eq:case1}
    & P(\Tilde{\mu}^m_i(t) - \sqrt{\frac{C_1\log t}{n_{m,i}(t)}} > \mu_i, n_{m,i}(t-1) \geq l) \notag \\
    & \leq \exp{\{-\frac{(\sqrt{\frac{C_1\log t}{n_{m,i}(t)}})^2}{2var(\Tilde{\mu}^m_i)}\}} \notag \\
    & \leq \exp{\{-\frac{(\sqrt{\frac{C_1\log t}{n_{m,i}(t)}})^2}{6|M_H|\frac{k_i\sigma^2}{n_{m,i}(t)}}\}} \notag \\
    & = \exp{\{-\frac{C_1\log t}{6|M_H|k_i\sigma^2}\}} \leq  \frac{1}{t^2}
\end{align}
where the first inequality holds by Chernoff bound, the second inequality is derived by plugging in the above upper bound on $var(\Tilde{\mu}^m_i(t))$, and the last inequality results from then choice of $C_1$ that satisfies $\frac{C_1}{6|M_H|k_i\sigma^2} \geq 1$. 

Likewise, by symmetry, we have 
\begin{align}\label{eq:case2}
    & P(\Tilde{\mu}^m_i(t) + \sqrt{\frac{C_1\log t}{n_{m,i}(t)}} < \mu_i, n_{m,i}(t-1) \geq l) \leq  \frac{1}{t^2}. 
\end{align}

Meanwhile, we have that \begin{align}\label{eq:case3}
    \sum_{t= L+1}^TP(\mu_i + 2\sqrt{\frac{C_1\log t}{n_{m,i}(t-1)}} > \mu_{i^*},n_{m,i}(t-1) \geq l) = 0
\end{align}
if the choice of $l$ satisfies $l \geq [\frac{4C_1\log T}{\Delta_i^2}]$ with $\Delta_i = \mu_{i^*} - \mu_i$.

Based on the decision rule, we have the following hold for $n_{m,i}(T)$ with $l \geq [\frac{4C_1\log T}{\Delta_i^2}]$, 
\begin{align*}
    n_{m,i}(T) & \leq l + \sum_{t=L+1}^T1_{\{a_t^m = i, n_{m,i}(t) > l\}}  \\
    & \leq l + \sum_{t=L+1}^T1_{\{\Tilde{\mu}^m_i - \sqrt{\frac{C_1\log t}{n_{m,i}(t-1)}} > \mu_i, n_{m,i}(t-1) \geq l\}} \\
    & \qquad \qquad + \sum_{t=L+1}^T1_{\{\Tilde{\mu}^m_{i^*} + \sqrt{\frac{C_1\log t}{n_{m,{i^*}}(t-1)}} < \mu_{i^*}, n_{m,i}(t-1) \geq l\}} \\
    & \qquad \qquad + \sum_{t=L+1}^T1_{\{\mu_i + 2\sqrt{\frac{C_1\log t}{n_{m,i}(t-1)}} > \mu_{i^*},n_{m,i}(t-1) \geq l\}}. 
\end{align*}

By taking the expectation over $n_{m,i}(t)$, we obtain 
\begin{align}\label{eq:n}
    E[n_{m,i}(t)] & \leq l + \sum_{t=L+1}^{T}P(\Tilde{\mu}^m_i(t) - \sqrt{\frac{C_1\log t}{n_{m,i}(t)}} > \mu_i, n_{m,i}(t-1) \geq l) \notag \\
    & \qquad \qquad + \sum_{t=L+1}^{T}P(\Tilde{\mu}^m_i(t) + \sqrt{\frac{C_1\log t}{n_{m,i}(t)}} < \mu_i, n_{m,i}(t-1) \geq l) \notag \\
    & \qquad \qquad + \sum_{t= L+1}^TP(\mu_i + 2\sqrt{\frac{C_1\log t}{n_{m,i}(t-1)}} > \mu_{i^*},n_{m,i}(t-1) \geq l) \notag \\
    & \leq l + \sum_{t=L+1}^{T}\frac{1}{t^2} + \sum_{t=L+1}^{T}\frac{1}{t^2} + 0 \notag \\
    & \leq l + \frac{\pi^2}{3} = [\frac{4C_1\log T}{\Delta_i^2}] + \frac{\pi^2}{3}
\end{align}
where the second inequality holds by using (\ref{eq:case1}), (\ref{eq:case2}), and (\ref{eq:case3}). 

Then by the definition of $T_{21}$, we derive 
\begin{align*}
    E[T_{21}|A] & = \sum_{m \in M_H}\sum_{k=1}^K\Delta_kE[n_{m,k}(t)] \\
    & \leq \sum_{m \in M_H}\sum_{k=1}^K\Delta_k([\frac{4C_1\log T}{\Delta_i^2}] + \frac{\pi^2}{3})
\end{align*}
where the inequality results from (\ref{eq:n}).

Consequently, we obtain 
\begin{align}
    (\ref{eq:E_T_2}) & \leq E[T_{21}|A] + |M_H|Kl^{T-1} \notag \\
    & \leq \sum_{m \in M_H}\sum_{k=1}^K\Delta_k([\frac{4C_1\log T}{\Delta_i^2}] + \frac{\pi^2}{3}) + |M_H|Kl^{T-1}.
\end{align}

Furthermore, we have 
\begin{align}
    (\ref{eq:R_T}) & \leq (c+1)\cdot L + E[T_1|A] + E[T_2|A]  \notag \\
    & \leq (c+1)\cdot L + \sum_{m \in M_H}\sum_{k=1}^K\Delta_k([\frac{4C_1\log T}{\Delta_i^2}] + \frac{\pi^2}{3}) + |M_H|Kl^{T-1} + 0 
\end{align}
which completes the proof.

\end{proof}

\subsection*{Proof of Theorem 2}
\begin{proof}
By the same deifnition of the regret, we, again, have the following regret decomposition

Note that 
\begin{align*}
    R_T 
    & = \max_i\sum_{m \in M_H}\sum_{t=1}^T\mu_i - \sum_{m \in M_H}\sum_{t=1}^T
(\mu_{a_m^t}^b - c_t) \\
    & = \max_i\sum_{m \in M_H}\sum_{t=1}^T\mu_i - \sum_{m \in M_H}\sum_{t=1}^T\mu_{a_m^t}^b + \sum_{m \in M_H}\sum_{t=1}^Tc_t \\
    & = \max_i\sum_{m \in M_H}\sum_{t=1}^T\mu_i - \sum_{m \in M_H}\sum_{t=1}^T\mu_{a_m^t}1_{b_t =1} + \sum_{m \in M_H}\sum_{t=1}^Tc_t \\
    & = \max_i\sum_{m \in M_H}\sum_{t=1}^T\mu_i - \sum_{m \in M_H}\sum_{t=1}^T\mu_{a_m^t}1_{b_t =1} + \sum_{m \in M_H}\sum_{t=1}^Tc_t1_{h_t = 1} \\
\end{align*}
Meanwhile, the regret can be bounded as follows
\begin{align}\label{eq:R_T}
    R_T & \leq  L + c\cdot L + \sum_{t=L+1}^T\sum_{m \in M_H}(\mu_{i^*} - \mu_{a_m^t}1_{b_t =1}) + \sum_{m \in M_H}\sum_{t=L+1}^Tc_t1_{h_t = 1} \notag \\
   &  \doteq (c+1)\cdot L + T_1 + T_2
\end{align}

We start with the second term $T_2$. Note that $h_t =1$ is equivalent to $\{m: m \in B_t \cap m \not\in M_H\} \neq \emptyset$. 

By taking the expectation over $T_2$, we derive
\begin{align*}
    E[T_2|A] & = \sum_{m \in M_H}\sum_{t=L+1}^TE[c_t \cdot 1_{h_t = 1}] \\
    & = \sum_{m \in M_H}\sum_{t=L+1}^TE[c_t \cdot 1_{\{m: m \in B_t \cap m \not\in M_H\} \neq \emptyset}] 
\end{align*}

By the Chernoff-Hoeffding's inequality and choosing $\eta_t \geq \frac{\sqrt{\log{t}}}{\sqrt{n_i(t)}}$, we obtain that 
\begin{align*}
    & P(|\bar{\mu}_{i}^m(t) - \mu_i| \geq \eta_t) \\
    & = P(|\bar{\mu}_{i}^m(t) - \mu_i| \geq \frac{\sqrt{\log{t}}}{\sqrt{n_i(t)}}) \\
    & \leq 2\exp{\{-\frac{\log{t}}{4\sigma^2n^2_{m,i}(t)}\}} \\
    & = 2\exp{\{-\frac{\log{t}}{4\sigma^2n^2_{m,i}(t)}\}} \\
    & \leq \textcolor{violet}{\frac{1}{t^2} \doteq P_t},
\end{align*}
when $t > L$, i.e. after the burn-in period. 

If $c_t \leq \frac{1}{t}$, then we have that 
 $E[T_2|A] \leq \log{T}$, which presents an upper bound on $T_2$. 

If $c_t = Dist(\Tilde{\mu}_i(t), \mu_i)$, then based on the definition of $B_t$ and $m \in B_t$ as in Option 2, we have that $\bar{\mu}_{i}^m(t)$ is smaller than the top $f$ values and larger than the below $f$ values. Based on Theorem 1 as in \citep{dong2023byzantine}, we have that 
\begin{align*}
    ||\hat{\mu}_i(t) - \bar{z}_i(t)|| \leq c_{\delta}\Delta^2
\end{align*}
where $\Delta$ represents the largest distance between the honest estimators and $\bar{z}_i(t)$ that is the averaged estimator maintained by all the honest participants. 

Then by definition, we obtain that 
\begin{align*}
    \Delta & = \max_{i \in M_H}|\bar{\mu}_i(t) - \bar{z}_i(t)|  \\
    & \leq \max_{i,j \in M_H} [|\bar{\mu}_i(t) - \mu_i| + |\bar{\mu}_j(t) - \mu_i|] \\
    & \leq 2\eta_t
\end{align*}
which holds with probability $1-P_t$. 

Therefore, we have that with probability $1-P_t$
\begin{align*}
    |\hat{\mu}_i(t) - \bar{z}_i(t)| \leq 2c_{\delta}\eta_t 
\end{align*}
and 
\begin{align*}
    |\bar{z}_i(t) - \mu_i| \leq \eta_t
\end{align*}
which holds by the Chernoff Bound inequality. 

Subsequently, we obtain that with probability $1-P_t$
\begin{align*}
    |\hat{\mu} - \mu_i| & \leq |\hat{\mu}_i(t) - \bar{z}_i(t)|+  |\bar{z}_i(t) - \mu_i| \\
    & \leq (2c_{\delta} + 1)\eta_t^6
\end{align*}

Meanwhile, for the distance measure, we have with probability $1-P_t$
\begin{align}\label{eq:dist}
    Dist(\tilde{\mu}_i(t) - \mu_i) &  = |\tilde{\mu}_i(t) - \mu_i|^6 \notag \\
    & = |\bar{q}_t\tilde{\mu}_i(t-1) + (1-\bar{q}_t)\hat{\mu}_i(t) - \mu_i|^6 \notag \\
    & \leq \bar{q}_t|\tilde{\mu}_i(t-1) - \mu_i|^6 + (1-\bar{q}_t)|\hat{\mu}_i(t) - \mu_i|^6 \notag \\
    & \leq \bar{q}_tDist(\Tilde{\mu}_{i}(t-1), \mu_i) + (1-\bar{q}_t)(2c_{\delta} + 1)^6\eta_t^6
\end{align}

Since  by definition, we derive that 
\begin{align*}
    & P(Dist(\tilde{\mu}_i(L), \mu_i) \geq O(\frac{\eta_t^2}{n_i(t)})) \\
    & \leq P(Dist(\tilde{\mu}_i(L), \mu_i) \geq O(\frac{\log{t}^3}{n_i(t)^3})) \\ 
    & \leq P(|\tilde{\mu}_i(L) - \mu_i| \geq O(\frac{\sqrt{\log{t}}}{\sqrt{n_i(t)}})) \\
    & \leq P(|\tilde{\mu}_i(L) - \mu_i| \geq \eta_t) \\
    & = P_t
\end{align*}

That is to say, with probability $1-P_t$, 
\begin{align*}
    Dist(\tilde{\mu}_i(L), \mu_i) \leq O(\frac{\eta_L^2}{n_i(L)})
\end{align*}

Next, suppose that at each time step $t$, with probability $1-P_t$, $Dist(\tilde{\mu}_i(t), \mu_i) \leq O(\frac{\eta_t^2}{n_i(t)})$.

Then by choosing $\bar{q}_t = 1 - \frac{1}{n_i(t)}$ and \ref{eq:dist}, we have that 
\begin{align*}
    & Dist(\Tilde{\mu}_{i}(t+1), \mu_i) \\
    & \leq \bar{q}_tDist(\Tilde{\mu}_{i}(t), \mu_i) + (1-\bar{q}_t)(2c_{\delta} + 1)^6\eta_t^6 \\
    & \leq O(\frac{\eta_t^2}{n_i(t)}) + O(\frac{1}{n_i(t)}\eta_t^6) \\
    & = O(\frac{\eta_{t+1}^2}{n_i(t+1)})
\end{align*}

Then we use the mathematical induction and  derive that for any $t \geq L$, with probability $1-P_t$, 
\begin{align*}
    Dist(\tilde{\mu}_i(t), \mu_i) \leq O(\frac{\eta_t^2}{n_i(t)}). 
\end{align*}

By the definition of cost, we obtain that with probability $1 - P_t$
\begin{align*}
    c_t & = \min_iDist(\tilde{\mu}_i(t), \mu_i) \\
    & \leq O(\frac{\log{t}}{\max_in_i(t)^2}) = O(\frac{\log{t}}{t^2})
\end{align*}
where the last inequality holds by the fact that $\max_{i}n_i(t) \geq \frac{\sum_in_i(t)}{K} = O(t)$.

Then we drive that 
\begin{align*}
     E[T_2|A] & \leq  \sum_{m \in M_H}\sum_{t=L+1}^TE[c_t \cdot 1_{\{m: m \in B_t \cap m \not\in M_H\} \neq \emptyset}] \\
     & \leq \sum_{m \in M_H}\sum_{t=L+1}^TE[c_t] \\
     & \leq \sum_{m \in M_H}\sum_{t=L+1}^T[(1-P_t) \cdot O(\frac{\log{t}}{t^2}) + P_t] \\
     & = \sum_{m \in M_H}\sum_{t=L+1}^TO(\frac{\log{t}}{t^2}) \\
     & \leq \log{T}\sum_{m \in M_H}\sum_{t=L+1}^TO(\frac{1}{t^2}) =  O(\log{T}). 
\end{align*}

We next follow the same steps as in the proof of Theorem 1 for bounding $E[T_1]$. Note that 
\begin{align*}
E[T_1|A] & \leq \sum_{t=L+1}^T\sum_{m \in M_H}(\mu_{i^*} - \mu_{a_m^t}1_{b_t =1}) \\
& = (T-L) \cdot |M_H|\cdot \mu_{i^*} - \sum_{m \in M_H}\sum_{t=L}^TE[\mu_{a_m^t}|b_t = 1]P(b_t = 1)  
\end{align*}

In the meantime, we obtain the following 
\begin{align*}
    & E[\mu_{a_m^t}|b_t = 1] \\
    & = E[\sum_{k=1}^K\mu_k\cdot 1_{a_m^t =k} | b_t = 1] \\
    & = \sum_{k=1}^KE[\mu_k1_{a_m^t = k}|b_t = 1] \\
    & \geq \sum_{k=1}^K\mu_k \cdot \frac{1}{P(b_t =1)} \cdot (E[1_{a_m^t = k}] - P(b_t = 0)).
\end{align*}

This immediately gives us that 
\begin{align}\label{eq:E_T_2}
    & E[T_1|A] \notag \\
    & \leq (T-L)|M_H|\mu_{i^*} - \sum_{m \in M_H}\sum_{t=L}^T(\sum_{k=1}^K\mu_k \cdot \frac{1}{P(b_t =1)} \cdot (E[1_{a_m^t = k}] - P(b_t = 0))) P(b_t = 1) \notag \\
    & = (T-L)|M_H|\mu_{i^*} - \sum_{m \in M_H}\sum_{t=L}^T(\sum_{k=1}^K\mu_k(E[1_{a_m^t = k}] - P(b_t = 0)) \notag \\
    & = (T-L)|M_H|\mu_{i^*} -\sum_{m \in M_H}\sum_{t=L}^T\sum_{k=1}^K \mu_kE[1_{a_m^t = k}] + \sum_{m \in M_H}\sum_{t=L}^T\sum_{k=1}^K \mu_kP(b_t = 0). 
\end{align}

Based on Theorem 2 in \citep{lamport2019byzantine}, the consensus is achieved, i.e. $b_t = 1$, as long as the digital signatures of the honest participants can not be forged. Based on our assumption, we have that the malicious participants can only perform existential forgery on the signatures of the honest participants and the attacks are adaptive chosen-message attack. Then based on the result, Main Theorem  in \citep{goldwasser1988digital}, the attack holds with probability at most $\frac{1}{Q(l)}$ for any polynomial function $Q$ and large enough $l$ where $l$ is the length of the signature. 

More precisely, we have that with probability at least $1- \frac{1}{l^T}$, the signature of the honest participants can not be forged, and thus, the consensus can be achieved, i.e.
\begin{align}\label{eq:b_t}
   P(b_t =1) \geq 1- \frac{1}{l^{T}}.  
\end{align}
Consequently, we have 
\begin{align}\label{eq:E_T_!_2}
   E[T_1] & \leq (T-L)|M_H|\mu_{i^*} -\sum_{m \in M_H}\sum_{t=L}^T\sum_{k=1}^K \mu_kE[1_{a_m^t = k}] + \sum_{m \in M_H}\sum_{t=L}^T\sum_{k=1}^K \mu_k(\frac{1}{l^T}) \notag \\
   & \leq \sum_{m \in M_H}\sum_{t=L}^T(\mu_{i^*} - \sum_{k=1}^K \mu_kE[1_{a_m^t = k}]) + |M_H|Kl^{T-1} \notag \\
   & = \sum_{m \in M_H}\sum_{k=1}^K\Delta_kE[n_{m,k}(t)] + |M_H|Kl^{T-1} 
\end{align}

Based on the decision rule, we have the following hold for $n_{m,i}(T)$ with $l \geq [\frac{4C_1\log T}{\Delta_i^2}]$, 
\begin{align*}
    n_{m,i}(T) & \leq l + \sum_{t=L+1}^T1_{\{a_t^m = i, n_{m,i}(t) > l\}}  \\
    & \leq l + \sum_{t=L+1}^T1_{\{\Tilde{\mu}^m_i - \sqrt{\frac{C_1\log t}{n_{m,i}(t-1)}} > \mu_i, n_{m,i}(t-1) \geq l\}} \\
    & \qquad \qquad + \sum_{t=L+1}^T1_{\{\Tilde{\mu}^m_{i^*} + \sqrt{\frac{C_1\log t}{n_{m,{i^*}}(t-1)}} < \mu_{i^*}, n_{m,i}(t-1) \geq l\}} \\
    & \qquad \qquad + \sum_{t=L+1}^T1_{\{\mu_i + 2\sqrt{\frac{C_1\log t}{n_{m,i}(t-1)}} > \mu_{i^*},n_{m,i}(t-1) \geq l\}}. 
\end{align*}

By taking the expectation over $n_{m,i}(t)$, we obtain 
\begin{align}\label{eq:n}
    E[n_{m,i}(t)] & \leq l + \sum_{t=L+1}^{T}P(\Tilde{\mu}^m_i(t) - \sqrt{\frac{C_1\log t}{n_{m,i}(t)}} > \mu_i, n_{m,i}(t-1) \geq l) \notag \\
    & \qquad \qquad + \sum_{t=L+1}^{T}P(\Tilde{\mu}^m_i(t) + \sqrt{\frac{C_1\log t}{n_{m,i}(t)}} < \mu_i, n_{m,i}(t-1) \geq l) \notag \\
    & \qquad \qquad + \sum_{t= L+1}^TP(\mu_i + 2\sqrt{\frac{C_1\log t}{n_{m,i}(t-1)}} > \mu_{i^*},n_{m,i}(t-1) \geq l) 
\end{align}

Recall that by our concentration inequality, we obtain that 
\begin{align*}
    & P(\Tilde{\mu}^m_i(t) + (\frac{C_1\log t}{n_{m,i}(t)})^{\frac{1}{6}} < \mu_i, n_{m,i}(t-1) \geq l)  \\
    & \leq P(|\Tilde{\mu}^m_i(t) - \mu_i| \geq O(\frac{\log{t}^{\frac{1}{6}}}{n_i(t)^{\frac{1}{3}}}), n_{m,i}(t-1) \geq l)  \\
    & = P(Dist(\Tilde{\mu}^m_i(t), \mu_i) \geq O(\frac{\eta_t^2}{n_i(t)}) , n_{m,i}(t-1) \geq l) \\
    & \leq P_t = \frac{1}{t^2}. 
\end{align*}

Meanwhile, we have that \begin{align}\label{eq:case3}
    \sum_{t= L+1}^TP(\mu_i + 2(\frac{C_1\log t}{n_{m,i}(t-1)})^{\frac{1}{6}} > \mu_{i^*},n_{m,i}(t-1) \geq l) = 0
\end{align}
if the choice of $l$ satisfies $l \geq [\frac{4C_1\log T}{\Delta_i^6}]$ with $\Delta_i = \mu_{i^*} - \mu_i$.

This immediately implies that 
\begin{align*}
    E[n_{m,i}(t)] & \leq l + \sum_{t= L+1}^TP_t + \sum_{t= L+1}^TP_t + 0 \\
    &  \leq l + \frac{\pi^2}{3} \\
    & = O(\log{T}).
\end{align*}

Then, by \ref{eq:E_T_!_2}, we arrive at 
\begin{align*}
    E[T_1] & \leq \sum_{m \in M_H}\sum_{k=1}^K\Delta_kE[n_{m,k}(t)] + |M_H|Kl^{T-1} \\
    & \leq O(\log{T}) + |M_H|Kl^{T-1}
\end{align*}

Henceforth, based on \ref{eq:R_T}, we have the following upper bound on the regret 
\begin{align}\label{eq:R_T}
    E[R_T|A] & \leq  (c+1)\cdot L + E[T_1|A] + E[T_2|A] \notag \\
    & \leq (c+1)\cdot L + O(\log{T}) + |M_H|Kl^{T-1}
\end{align}
which completes the proof.

\subsection*{Proof of Theorem 3}
\begin{proof}
    Again, we start by decomposing the regret as 
    \begin{align*}
            R_T & \leq  L + c\cdot L + \sum_{t=L+1}^T\sum_{m \in M_H}(\mu_{i^*} - \mu_{a_m^t}1_{b_t =1}) + \sum_{m \in M_H}\sum_{t=L+1}^Tc_t1_{h_t = 1} \notag \\
   &  \doteq (c+1)\cdot L + T_1 + T_2
    \end{align*}

We note that the consensus protocol runs $M$ times, with each validator (i.e., participant in this case) being selected as a commander. For any malicious participant $j \in M_A^2$, it serves as a commander and is thus included in $D_t$. This holds true because, according to Lemma 3 in \citep{goldwasser1988digital}, if the message is a chandelier tree generated by the secret key $SK_m$ of participant $m$, any participant can verify the public key $PK_m$, or equivalently, trace back to the root of the signature tree of the message sender. Due to the unique mapping between $PK_m$ and $m$, the honest participant keeps a record of the vertex index of the malicious participants that attack the consensus.

This implies that  $j \not\in B_t$, i.e. the set $B_t$ can only contain estimators from either honest participants or set $M_A^1$ that satisfies $|M_A^1| < M_H - 1$. Therefore, the property of $B_t$ follows from that as in Option 2, which essentially indicates that Option 3 is equivalent to Option 2 with at least one half honest participants. Considering that the remaining algorithmic steps are the same, the analysis of $T_1$ and $T_2$ is consistent with that of Theorem 2. 

Consequently, we have that 
\begin{align*}
    E[T_1] & \leq \sum_{m \in M_H}\sum_{k=1}^K\Delta_kE[n_{m,k}(t)] + |M_H|Kl^{T-1} \\
    & \leq O(\log{T}) + |M_H|Kl^{T-1}
\end{align*}
and 
\begin{align*}
    E[T_2] \leq \log{T}\sum_{m \in M_H}\sum_{t=L+1}^TO(\frac{1}{t^2}) =  O(\log{T}). 
\end{align*}

Subsequently, we derive the same regret bound as in Theorem, as 
\begin{align*}
    E[R_T] & \leq (c+1)\cdot L + O(\log{T}) + |M_H|Kl^{T-1} + \log{T}\sum_{m \in M_H}\sum_{t=L+1}^TO(\frac{1}{t^2}) \\
    & = O(\log{T})
\end{align*}
which completes the proof. 

\end{proof}

\subsection*{Proof of Theorem 4}
\begin{proof}
The proof of Theorem 4 is similar to that of Theorem 3 as follows. The regret of the coordination mechanism is again decomposed as 
    \begin{align*}
            R_T & \leq  L + c\cdot L + \sum_{t=L+1}^T\sum_{m \in M_H}(\mu_{i^*} - \mu_{a_m^t}1_{b_t =1}) + \sum_{m \in M_H}\sum_{t=L+1}^Tc_t1_{h_t = 1} \notag \\
   &  \doteq (c+1)\cdot L + T_1 + T_2
    \end{align*}

For malicious participant $j \in M_A^2$, it only attacks the consensus process and does not attack the estimators. In the meantime, for malicious participant $l \in M_A^1$, it only attacks the estimators, but does not attack the consensus process. Since $|M_A^1| < \frac{1}{2}M - 1$, when using Option 2, the set $B_t$ is the same as the case where only at most $\frac{1}{2}M$ participants are malicious. Therefore, we have that
\begin{align*}
    E[T_2] \leq \log{T}\sum_{m \in M_H}\sum_{t=L+1}^TO(\frac{1}{t^2}) =  O(\log{T}). 
\end{align*}

Meanwhile, since the total number of malicious participants in $M_A^1$ meets that $|M_A^1| < \frac{1}{2}M - 1$, and the consensus protocol runs $M$ participants with each participant as a commander, the consensus always succeeds with probability at least $1 - \frac{1}{l^T}$. This immediately gives us that based on (\ref{eq:E_T_!_2})
\begin{align*}
E[T_1] & = \sum_{m \in M_H}\sum_{k=1}^K\Delta_kE[n_{m,k}(t)] + |M_H|Kl^{T-1} 
\end{align*}

Meanwhile, the statistical property of $n_{m,k}(t)$ depends on that of the global estimator $\Tilde{\mu}_k(t)$ by our decision and update rule. The computation of $\Tilde{\mu}_k(t)$ depends on set $B_t$, which is the same as the case where there are only at most $\frac{1}{2}M - 1$ malicious participants. Subsequently, we obtain
\begin{align*}
    E[n_{m,i}(t)] & \leq l + \sum_{t= L+1}^TP_t + \sum_{t= L+1}^TP_t + 0 \\
    &  \leq l + \frac{\pi^2}{3} \\
    & = O(\log{T}).
\end{align*}
when $l \geq [\frac{4C_1\log T}{\Delta_i^6}]$ with $\Delta_i = \mu_{i^*} - \mu_i$.

Then, based on \ref{eq:E_T_!_2}, we again arrive at 
\begin{align*}
    E[T_1] & \leq \sum_{m \in M_H}\sum_{k=1}^K\Delta_kE[n_{m,k}(t)] + |M_H|Kl^{T-1} \\
    & \leq O(\log{T}) + |M_H|Kl^{T-1}
\end{align*}

Henceforth, by the regret decomposition, we have the following upper bound on the regret 
\begin{align}\label{eq:R_T}
    E[R_T|A] & \leq  (c+1)\cdot L + E[T_1|A] + E[T_2|A] \notag \\
    & \leq (c+1)\cdot L + O(\log{T}) + |M_H|Kl^{T-1}
\end{align}
which completes the proof.

\end{proof}

\subsection*{Proof of Theorem 5}
\begin{proof}
The proof of Theorem 5 follows a similar approach to that of Theorem 4. the coordination mechanism's regret can be decomposed as follows:    \begin{align*}
            R_T & \leq  L + c\cdot L + \sum_{t=L+1}^T\sum_{m \in M_H}(\mu_{i^*} - \mu_{a_m^t}1_{b_t =1}) + \sum_{m \in M_H}\sum_{t=L+1}^Tc_t1_{h_t = 1} \notag \\
   &  \doteq (c+1)\cdot L + T_1 + T_2
    \end{align*}

For malicious participant $j \in M_A^2$, the attacks are limited to the consensus process and do not affect the estimators. Conversely, a malicious participant $l \in M_A^1$, it targets the estimators but does not disrupt the consensus process. Given that $|M_A^1| < \frac{1}{2}M - 1$, when using Option 2, the set $B_t$ is the same as the case where only at most $\frac{1}{2}M$ participants are malicious. Therefore, we have that
\begin{align*}
    E[T_2] \leq \log{T}\sum_{m \in M_H}\sum_{t=L+1}^TO(\frac{1}{t^2}) =  O(\log{T}). 
\end{align*}

The analysis of $T_1$ requires further work, especially considering the development of this new commander selection protocol. More specifically, by definition, we have $w_m(t) = w_{m} = 1 - \frac{\log{T}}{T}$, for any $m \in M_H$. Consider the event of whether honest participant $m$ is selected as a commander as $E_{m}^t$. In other words, $E_m^t =1$ if participant $m$ is a commander and 0 otherwise. Define $E_t$ as $\cap_{m \in M_H}\{E_{m}^t = 0\}$. Then we have that 
    \begin{align*}
        E[\sum_{t=1}^TE_t] & = \sum_{t=1}^TE[\cap_{m \in M_H}\{E_{m}^t = 0\}] \\
        & \leq \sum_{t=1}^T\sum_{m \in M_H}E[\{E_{m}^t = 0\}] \\
        & = \sum_{t=1}^T\sum_{m \in M_H}(1 - w_m(t)) = \log{T}.
    \end{align*}
    It implies that for the total length of having no honest commanders is at most $\log{T}$, there is no honest commander, which indicated that the consensus fails. In the meantime, we note that if there is a honest commander in set $S_C(t)$, then the consensus is achieved with the correct $\Tilde{\mu}$, i.e. $b_t=1$, and thus we have $E[1_{b_t = 0}] \leq \frac{\log{T}}{T}$ and $E[\sum_{t=1}^T1_{b_t = 0}] \leq \log{T}$.  
    
    Differently, by our choice, $w_j(t) = w_{j} = \frac{\log{\frac{|M_A|}{\eta}}}{T}$, for any $j \in M_A$. Then we consider the event of whether malicious participant $j$ is selected as a commander or not, namely, $F^j_t$. Likewise, $F^j_t = 1$ if participant $j$ is a commander and $0$ otherwise. Define $F_t = \cap_{j \in M_A}\{\exists s \leq t, s.t. F^j_s = 1\}$. Then we obtain 
    \begin{align*}
      P(F_t) & = P(\cap_{j \in M_A}\{\exists s \leq t, s.t. F^j_s = 1\}) \\
      &  \geq 1 - \sum_{j \in M_A}P(\{\forall s \leq t, s.t. F^j_s = 0\}) \\
      & = 1 - \sum_{j \in M_A}(1 - w_j)^t \\
      & = 1 - |M_A|(1 - w_j)^t \\
      & \geq 1 - |M_A|e^{-w_jt}
    \end{align*}
    By the choice of $w_j = \frac{\log{\frac{|M_A|}{\eta}}}{T}$, we derive that $P(F_{t}) \geq 1 - |M_A|e^{-w_jt} = 1 - \eta$. This means that at each time step, the malicious participants have high probability of being chosen as commanders, which provides enough incentive for them to participate, and thus implies the rationality of this probability.

Subsequently, since the total number of malicious participants in $M_A^1$ meets that $|M_A^1| < \frac{1}{2}M - 1$, and the consensus protocol runs $M$ participants with at least one honest commander, the consensus always succeeds with probability at least $1 - \frac{\log{T}}{T}$. Based on \ref{eq:E_T_2}, we obtain that 
\begin{align*}
E[T_1] & \leq (T-L)|M_H|\mu_{i^*} -\sum_{m \in M_H}\sum_{t=L}^T\sum_{k=1}^K \mu_kE[1_{a_m^t = k}] + \sum_{m \in M_H}\sum_{t=L}^T\sum_{k=1}^K \mu_kP(b_t = 0) \\
& \leq \sum_{m \in M_H}\sum_{k=1}^K\Delta_kE[n_{m,k}(t)] + |M_H|KO(\log{T}). 
\end{align*}

Again, based on the decision rule, we have the following hold for $n_{m,i}(T)$ with $l \geq [\frac{4C_1\log T}{\Delta_i^2}]$, 
\begin{align*}
    n_{m,i}(T) & \leq l + \sum_{t=L+1}^T1_{\{a_t^m = i, n_{m,i}(t) > l\}}  \\
    & \leq l + \sum_{t=L+1}^T1_{\{\Tilde{\mu}^m_i - \sqrt{\frac{C_1\log t}{n_{m,i}(t-1)}} > \mu_i, n_{m,i}(t-1) \geq l\}} \\
    & \qquad \qquad + \sum_{t=L+1}^T1_{\{\Tilde{\mu}^m_{i^*} + \sqrt{\frac{C_1\log t}{n_{m,{i^*}}(t-1)}} < \mu_{i^*}, n_{m,i}(t-1) \geq l\}} \\
    & \qquad \qquad + \sum_{t=L+1}^T1_{\{\mu_i + 2\sqrt{\frac{C_1\log t}{n_{m,i}(t-1)}} > \mu_{i^*},n_{m,i}(t-1) \geq l\}}. 
\end{align*}

Note that taking the expectation over $n_{m,i}(t)$ gives 
\begin{align}
    E[n_{m,i}(t)] & \leq l + \sum_{t=L+1}^{T}P(\Tilde{\mu}^m_i(t) - \sqrt{\frac{C_1\log t}{n_{m,i}(t)}} > \mu_i, n_{m,i}(t-1) \geq l) \notag \\
    & \qquad \qquad + \sum_{t=L+1}^{T}P(\Tilde{\mu}^m_i(t) + \sqrt{\frac{C_1\log t}{n_{m,i}(t)}} < \mu_i, n_{m,i}(t-1) \geq l) \notag \\
    & \qquad \qquad + \sum_{t= L+1}^TP(\mu_i + 2\sqrt{\frac{C_1\log t}{n_{m,i}(t-1)}} > \mu_{i^*},n_{m,i}(t-1) \geq l) 
\end{align}

Using the concentration inequality, we obtain that 
\begin{align*}
    & P(\Tilde{\mu}^m_i(t) + (\frac{C_1\log t}{n_{m,i}(t)})^{\frac{1}{6}} < \mu_i, n_{m,i}(t-1) \geq l)  \\
    & \leq P(|\Tilde{\mu}^m_i(t) - \mu_i| \geq O(\frac{\log{t}^{\frac{1}{6}}}{n_i(t)^{\frac{1}{3}}}), n_{m,i}(t-1) \geq l)  \\
    & = P(Dist(\Tilde{\mu}^m_i(t), \mu_i) \geq O(\frac{\eta_t^2}{n_i(t)}) , n_{m,i}(t-1) \geq l) \\
    & \leq P_t = \frac{1}{t^2}. 
\end{align*}

Likewise, we obtain that 
\begin{align}
    \sum_{t= L+1}^TP(\mu_i + 2(\frac{C_1\log t}{n_{m,i}(t-1)})^{\frac{1}{6}} > \mu_{i^*},n_{m,i}(t-1) \geq l) = 0
\end{align}
if the choice of $l$ satisfies $l \geq [\frac{4C_1\log T}{\Delta_i^6}]$ with $\Delta_i = \mu_{i^*} - \mu_i$, which leads to 
\begin{align*}
    E[n_{m,i}(t)] & \leq l + \sum_{t= L+1}^TP_t + \sum_{t= L+1}^TP_t + 0 \\
    &  \leq l + \frac{\pi^2}{3} \\
    & = O(\log{T}).
\end{align*}

Consequently, we obtain that 
\begin{align*}
    E[T_1] & \leq \sum_{m \in M_H}\sum_{k=1}^K\Delta_kE[n_{m,k}(t)] + |M_H|Kl^{T-1} \\
    & \leq O(\log{T}) + |M_H|Kl^{T-1}
\end{align*}

Combining all these together, we derive the following upper bound on the expected regret 
\begin{align}\label{eq:R_T}
    E[R_T|A] & \leq  (c+1)\cdot L + E[T_1|A] + E[T_2|A] \notag \\
    & \leq (c+1)\cdot L + O(\log{T}). 
\end{align}

This concludes the proof of Theorem 5.

\end{proof}

\subsection*{Proof of Theorem 6}
\begin{proof}
Again, we decompose system's regret as follows:    \begin{align*}
            R_T & \leq  L + c\cdot L + \sum_{t=L+1}^T\sum_{m \in M_H}(\mu_{i^*} - \mu_{a_m^t}1_{b_t =1}) + \sum_{m \in M_H}\sum_{t=L+1}^Tc_t1_{h_t = 1} \notag \\
   &  \doteq (c+1)\cdot L + T_1 + T_2
    \end{align*}

Differently, the definition of $c_t$ is a constant-based one, where $c_t = c1_{\exists m \in C_t \& m\in M_A^1}$ since the estimators in $C_t$ are used for computing $\Tilde{\mu}_i^m(t)$. Note that here we do not count malicious participants in $M_A^2$ in, as these participant do not perform attacks on the estimators, i.e. having no negative effect on $\Tilde{\mu}_{i}(t)$. 

In the meantime, by the robust estimator property of the estimators in $B_t$, we obtain that 
\begin{align*}
    ||\hat{\mu}_i(t) - \bar{z}_i(t)|| \leq c_{\Delta}\Delta^2
\end{align*}
where with probability $1-P_t$,
\begin{align*}
    \Delta & = \max_{m \in M_H}|\bar{\mu}_i^m(t) - \bar{z}_i(t)|  \\
    & \leq \max_{m,j \in M_H} [|\bar{\mu}_i^j(t) - \mu_i| + |\bar{\mu}_i^m(t) - \mu_i|] \\
    & \leq 2\eta_t 
\end{align*}

This immediately implies that for $m \in M_H$
\begin{align*}
    |\bar{\mu}^m_i(t) - \hat{\mu}_i| 
    & \leq |\bar{\mu}_i(t) - \bar{z}_i(t) + \bar{z}_i(t) - \hat{\mu}_i| \\
    & \leq 2\eta_t + 4(c_{\Delta})\eta_t^2 \\
    & \leq \frac{1}{2}\epsilon||q||
\end{align*}
where the last inequality holds by the choice of $\epsilon$ and $||q||$ denotes the minimum value of the random variable following distribution $q_i^m$.

By assumption, we have that for $j \in M_A^1$,  
\begin{align*}
    f_i^j(t) = (1-\epsilon)g_i^m(t) + \epsilon q_i^m(t)
\end{align*}
where $f_i^j(t)$ represents the underlying distribution of the rewards of malicious agent $j \in M_A^1$. It is worth emphasizing that this assumption is consistent with \citep{dubey2020private}, originated from the Huber's $\epsilon$-Contamination model \citep{huber2011robust}. 

By taking the expectation over the distributions, we obtain that 
\begin{align*}
    \mu_j = (1-\epsilon)\mu_i + \epsilon E[q]
\end{align*}

This implies that for $j \in M_A^1$ with probability $1-2P_t$
\begin{align*}
    |\bar{\mu}^j_i(t) - \hat{\mu}_i| & \geq 
    |\bar{\mu}^j_i(t) - \bar{\mu}^m_i(t) + \bar{\mu}^m_i(t)-\hat{\mu}_i| \\
    & \geq |\bar{\mu}^j_i(t) - \bar{\mu}^m_i(t)| - |\bar{\mu}^m_i(t)-\hat{\mu}_i| \\
    & \geq \epsilon||q|| - \frac{1}{2}\epsilon||q|| \\
    & \geq \frac{1}{2}\epsilon||q||
\end{align*}

This is to say that $j \in M_A^1$ does not belong to $C_t$, and thus implies that $c_t = 0$ for $t > L$ with probability $1 - 3P_t = 1 - \frac{3}{t^2}$, and $c_t = c$ with probability $\frac{3}{t^2}$. 

Therefore we have that 
\begin{align*}
    E[T_2] \leq \sum_{m \in M_H}\sum_{t=L+1}^TO(\frac{3}{t^2}) =  O(1). 
\end{align*}

Based on (\ref{eq:n}), we again obtain that  
\begin{align}
    E[n_{m,i}(t)] & \leq l + \sum_{t=L+1}^{T}P(\Tilde{\mu}^m_i(t) - \sqrt{\frac{C_1\log t}{n_{m,i}(t)}} > \mu_i, n_{m,i}(t-1) \geq l) \notag \\
    & \qquad \qquad + \sum_{t=L+1}^{T}P(\Tilde{\mu}^m_i(t) + \sqrt{\frac{C_1\log t}{n_{m,i}(t)}} < \mu_i, n_{m,i}(t-1) \geq l) \notag \\
    & \qquad \qquad + \sum_{t= L+1}^TP(\mu_i + 2\sqrt{\frac{C_1\log t}{n_{m,i}(t-1)}} > \mu_{i^*},n_{m,i}(t-1) \geq l) 
\end{align}

By the fact that with probability $1-3P_t$, $c_t = 0$, we again have that the validated estimator $\Tilde{\mu}_{i}(t)$ can be expressed as with probability $1-3P_t$
\begin{align*}
    \Tilde{\mu}_{i}(t) = \sum_{j \in A_t \cap M_H}w_{j,i}(t)\bar{\mu}_{i}^j(t)
\end{align*}
which is also equivalent to $\Tilde{\mu}_{i}^m(t)$. Here the weight $w_{j,i}(t)$ meets the condition 
\begin{align*}
    \sum_{j \in A_t \cap M_H}w_{j,i}(t) = 1, 
\end{align*}
which immediately implies that 
\begin{align*}
    E[\Tilde{\mu}_{i}(t)] = \mu_i.
\end{align*}

We note that the variance of $\Tilde{\mu}_{i}(t)$, $var(\Tilde{\mu}_{i}(t))$, satisfies that, with probability $1-3P_t$
\begin{align*}
    var(\Tilde{\mu}_{i}(t)) & = var(\sum_{j \in A_t \cap M_H}w_{j,i}(t)\bar{\mu}_{i}^j(t)) \\
    & \leq |A_t \cap M_H|\sum_{j \in A_t \cap M_H}w_{j,i}(t)^2var(\bar{\mu}_{i}^j(t))) \\
    & \leq |A_t \cap M_H|\sum_{j \in A_t \cap M_H}w_{j,i}^2(t)\sigma^2\frac{1}{n_{j,i}(t)} \\
    & \leq |A_t \cap M_H|\sum_{j \in A_t \cap M_H}w_{j,i}^2(t)\sigma^2\frac{k_i}{n_{m,i}(t)} \\
    & = |M_H|\frac{k_i}{n_{m,i}(t)}\sum_{j \in M_H}w_{j,i}^2(t)\sigma^2 \\
    & \leq |M_H|\frac{k_i\sigma^2}{n_{m,i}(t)}
\end{align*}
where the inequality holds by the Cauchy-Schwarz inequality, the second inequality holds by the definition of sub-Gaussian distributions, the third inequality results from the construction of $A_t$, and the last inequality is as a result of $(a+b)^2 \geq a^2+b^2$. 

Subsequently, we have that 
\begin{align}
    & P(\Tilde{\mu}^m_i(t) - \sqrt{\frac{C_1\log t}{n_{m,i}(t)}} > \mu_i, n_{m,i}(t-1) \geq l) \notag \\
    & \leq \exp{\{-\frac{(\sqrt{\frac{C_1\log t}{n_{m,i}(t)}})^2}{2var(\Tilde{\mu}^m_i)}\}} \notag \\
    & \leq (\exp{\{-\frac{(\sqrt{\frac{C_1\log t}{n_{m,i}(t)}})^2}{2|M_H|\frac{k_i\sigma^2}{n_{m,i}(t)}}\}})(1-3P_t) + 3P_t\notag \\
    & = (1-3P_t)\exp{\{-\frac{C_1\log t}{2|M_H|k_i\sigma^2}\}} + 3P_t \leq  \frac{4}{t^2}
\end{align}
where the first inequality holds by Chernoff bound, the second inequality is derived by plugging in the above upper bound on $var(\Tilde{\mu}^m_i(t))$, and the last inequality results from then choice of $C_1$ that satisfies $\frac{C_1}{6|M_H|k_i\sigma^2} \geq 1$. 

Likewise, by symmetry, we have 
\begin{align}
    & P(\Tilde{\mu}^m_i(t) + \sqrt{\frac{C_1\log t}{n_{m,i}(t)}} < \mu_i, n_{m,i}(t-1) \geq l) \leq  \frac{4}{t^2}. 
\end{align}

This immediately implies that 
\begin{align*}
    E[n_{m,i}(t)] & \leq l + \sum_{t= L+1}^T4P_t + \sum_{t= L+1}^T4P_t + 0 \\
    &  \leq l + \frac{4\pi^2}{3} \\
    & = O(\log{T}).
\end{align*}

Then we arrive at 
\begin{align*}
    E[T_1] & \leq \sum_{m \in M_H}\sum_{k=1}^K\Delta_kE[n_{m,k}(t)] + |M_H|Kl^{T-1} \\
    & \leq O(\log{T}) + |M_H|Kl^{T-1}
\end{align*}

Once again, by the regret decomposition, we obtain that 
\begin{align*}
   E[R_T]  & \leq E[(c+1)\cdot L + T_1 + T_2] \\
   & \leq (c+1)\cdot L + O(1) + O(\log{T}) + |M_H|Kl^{T-1} \\
   & = O(\log{T})
\end{align*}

\end{proof}

\subsection*{Proof of Theorem 7}

\begin{proof}
As before, the regret is decomposed as    \begin{align*}
    R_T & \leq  L + c\cdot L + \sum_{t=L+1}^T\sum_{m \in M_H}(\mu_{i^*} - \mu_{a_m^t}1_{b_t =1}) + \sum_{m \in M_H}\sum_{t=L+1}^Tc_t1_{h_t = 1} \notag \\
   &  \doteq (c+1)\cdot L + T_1 + T_2
    \end{align*}

We first show the monotonicity of the reputation score after the burn-in period. Recall that the reputation score of participant $i$ is defined as 
\begin{align*}
    U_i^t & = \sum_{j=1}^{K}-(\bar{\mu}_j^i(t) - \Tilde{\mu}_j(t))^2 - \epsilon^2e^{(\overset{\Delta}{\mu}_j^i(t)- \Tilde{\mu}_j(t)^2)} \\
    & \doteq U_i^{1,t} + U_i^{2,t}
\end{align*}
where $\overset{\Delta}{\mu}_j^i(t)$ denotes the estimator for arm $j$ given by participant $i$ after the consensus step, and $\bar{\mu}_j^i(t), \Tilde{\mu}_j(t)$ are the aforementioned estimators for arm $j$.


We consider $t > L$, where 
\begin{align}
    & P(\Tilde{\mu}^m_i(t) + \sqrt{\frac{C_1\log t}{n_{m,i}(t)}} < \mu_i, n_{m,i}(t-1) \geq l) \leq  \frac{1}{t^2}. 
\end{align}

We consider malicious participant $j \in M_A^1$ and honest participant $m \in M_H$, and by definition, it only attacks the estimators, which immediately gives us that 
\begin{align*}
    U_j^{2,t} = U_i^{2,t} = 0.
\end{align*}

Again, by this definition and the pre-fixed $\epsilon$ zone, we obtain  
\begin{align*}
    f_i^j(t) = (1-\epsilon)g_i^m(t) + \epsilon q_i^m(t)
\end{align*}
and thus 
$j \in M_A^1$ with probability $1-2P_t$
\begin{align*}
    |\bar{\mu}^j_i(t) - \hat{\mu}_i| & \geq 
    |\bar{\mu}^j_i(t) - \bar{\mu}^m_i(t) + \bar{\mu}^m_i(t)-\hat{\mu}_i| \\
    & \geq |\bar{\mu}^j_i(t) - \bar{\mu}^m_i(t)| - |\bar{\mu}^m_i(t)-\hat{\mu}_i| \\
    & \geq \epsilon||q|| - \frac{1}{2}\epsilon||q|| \\
    & \geq \frac{1}{2}\epsilon||q||
\end{align*}

Subsequently, we arrive at 
\begin{align*}
    |\bar{\mu}^j_i(t) - \Tilde{\mu}_i| \geq \frac{1}{2}\epsilon||q||
\end{align*}

Meanwhile, we have 
\begin{align*}
    |\bar{\mu}^m_i(t) - \hat{\mu}_i| 
    & \leq |\bar{\mu}_i(t) - \bar{z}_i(t) + \bar{z}_i(t) - \hat{\mu}_i| \\
    & \leq 2\eta_t + 4(c_{\Delta})\eta_t^2 \\
    & \leq \frac{1}{2}\epsilon||q||
\end{align*}
which also implies that 
\begin{align*}
    |\bar{\mu}^m_i(t) - \Tilde{\mu}_i| \geq \frac{1}{2}\epsilon||q||
\end{align*}

That is to say, the first term in the reputation score meets that
\begin{align*}
    U_j^{1,t} \leq U_{m}^{1,t}
\end{align*}
and subsequently, we obtain 
\begin{align*} 
    U_j^t \leq  U_m^t.
\end{align*}

Next, let us consider malicious participant $k \in M_A^2$ and honest participant $m \in M_H$. By definition, participant $k$ only attacks the consensus process without altering the estimators. However, Equivalently, this does not imply 
\begin{align*}
    U_k^{1,t} = U_m^{1,t} = 0
\end{align*}
since $\bar{\mu}_i^m \neq \bar{\mu}_i^k$ due to the randomness, which brings additional challenge. 

We consider the difference between the estimators, 
\begin{align*}
   & |U_k^{1,t} - U_m^{1,t}| \\
   & \leq |\bar{\mu}_i^k(t) - \Tilde{\mu}_i(t)|^2 + |\bar{\mu}_i^m(t)-\Tilde{\mu}_i(t)|^2 \\
   & \leq \frac{1}{2}(\epsilon ||q||)^2.
\end{align*}

In the meantime, if participant $k$ serves as a validator, we immediately have 
\begin{align*}
    & \overset{\Delta}{\mu}_j^i(t)- \Tilde{\mu}_j(t)^2) > 0, \\
    & U_k^{2,t} < -\epsilon^2
\end{align*}
while in the meantime, $U_m^{2,t} = 0$.

Consequently, we arrive at 
\begin{align*}
   U_k^t - U_m^t & = U_k^{1,t} - U_m^{1,t} +  U_k^{2,t} - U_m^{2,t} \\
   & \leq |U_k^{1,t} - U_m^{1,t}| + U_k^{2,t} \\
   & \leq \frac{1}{2}\epsilon^2 - \epsilon^2 = -\frac{1}{2}\epsilon^2 < 0
\end{align*}
where the second last inequality holds by assuming $||q|| \leq 1 $ without loss of generality. 

Combining these all together, we obtain that
\begin{align*}
    U_j^t < U_m^t 
\end{align*}
for any malicious participant $j \in M_A$ and honest participant $m \in M_H$, which implies the monotonocity of $U$ quantity in the reputation score. 

Subsequently, by the monotone preserving property of function $G(\cdot)$, we immediately have 
\begin{align*}
    G(U_j^{t}) < G(U_m^t)
\end{align*}
for any malicious participant $j \in M_A$ and honest participant $m \in M_H$. 

Based on the Validator selection Protocol where the top $N$ participants are selected with $|M_H| < N < 2|M_H|-1$, we obtain that $M_H \subset S_V(t)$ and $|S_V(t)| \leq 2|M_H|-1$, which implies that the consensus always achieves if every validator is selected as a commander for exactly once, i.e. $b_t =1$ with probability at most $1 - Ml^{-T}$.

Otherwise, if a participant $k \in M_A^2$ is never selected as a validator, then the set of validators does not contain any malicious participants attacking the consensus, and then the consensus always achieves, i.e. $b_t = 1$. 

To summarize, we have that 
\begin{align*}
  P(b_t = 1) \geq 1 - Ml^{-T}. 
\end{align*}

Note that the set $B_t, C_t$ herein is the same as the set of $B_t, C_t$ as in Theorem 6, which immediately implies that $j \in M_A^1$ does not belong to $C_t$, and thus implies that $c_t = 0$ for $t > L$ with probability $1 - 3P_t = 1 - \frac{3}{t^2}$, and $c_t = c$ with probability $\frac{3}{t^2}$. 

Therefore we again obtain that by the definition of $T_2$ that depends on $b_t$ and $c_t$
\begin{align*}
    E[T_2] \leq \sum_{m \in M_H}\sum_{t=L+1}^TO(\frac{3}{t^2}) =  O(1). 
\end{align*}

Again, using (\ref{eq:n}), we obtain the following decomposition 
\begin{align}
    E[n_{m,i}(t)] & \leq l + \sum_{t=L+1}^{T}P(\Tilde{\mu}^m_i(t) - \sqrt{\frac{C_1\log t}{n_{m,i}(t)}} > \mu_i, n_{m,i}(t-1) \geq l) \notag \\
    & \qquad \qquad + \sum_{t=L+1}^{T}P(\Tilde{\mu}^m_i(t) + \sqrt{\frac{C_1\log t}{n_{m,i}(t)}} < \mu_i, n_{m,i}(t-1) \geq l) \notag \\
    & \qquad \qquad + \sum_{t= L+1}^TP(\mu_i + 2\sqrt{\frac{C_1\log t}{n_{m,i}(t-1)}} > \mu_{i^*},n_{m,i}(t-1) \geq l) 
\end{align}

Furthermore, with probability $1-3P_t$, $c_t = 0$ again implies that the validated estimator $\Tilde{\mu}_{i}(t)$ has the following explicit formula, with probability $1-3P_t$
\begin{align*}
    \Tilde{\mu}_{i}(t) = \sum_{j \in A_t \cap M_H}w_{j,i}(t)\bar{\mu}_{i}^j(t)
\end{align*}
which is the value of $\Tilde{\mu}_{i}^m(t)$ as well, where $w_{j,i}(t)$ are the weights such that  
\begin{align*}
    \sum_{j \in A_t \cap M_H}w_{j,i}(t) = 1. 
\end{align*}
This immediately gives us that 
\begin{align*}
    E[\Tilde{\mu}_{i}(t)] = \mu_i.
\end{align*}

We note that the variance of $\Tilde{\mu}_{i}(t)$, $var(\Tilde{\mu}_{i}(t))$, satisfies that, with probability $1-3P_t$
\begin{align*}
    var(\Tilde{\mu}_{i}(t)) & = var(\sum_{j \in A_t \cap M_H}w_{j,i}(t)\bar{\mu}_{i}^j(t)) \\
    & \leq |A_t \cap M_H|\sum_{j \in A_t \cap M_H}w_{j,i}(t)^2var(\bar{\mu}_{i}^j(t))) \\
    & \leq |A_t \cap M_H|\sum_{j \in A_t \cap M_H}w_{j,i}^2(t)\sigma^2\frac{1}{n_{j,i}(t)} \\
    & \leq |A_t \cap M_H|\sum_{j \in A_t \cap M_H}w_{j,i}^2(t)\sigma^2\frac{k_i}{n_{m,i}(t)} \\
    & = |M_H|\frac{k_i}{n_{m,i}(t)}\sum_{j \in M_H}w_{j,i}^2(t)\sigma^2 \\
    & \leq |M_H|\frac{k_i\sigma^2}{n_{m,i}(t)}
\end{align*}
where the inequality holds by the Cauchy-Schwarz inequality, the second inequality holds by the definition of sub-Gaussian distributions, the third inequality results from the construction of $A_t$, and the last inequality is as a result of $(a+b)^2 \geq a^2+b^2$. 

Subsequently, we have that 
\begin{align}
    & P(\Tilde{\mu}^m_i(t) - \sqrt{\frac{C_1\log t}{n_{m,i}(t)}} > \mu_i, n_{m,i}(t-1) \geq l) \notag \\
    & \leq \exp{\{-\frac{(\sqrt{\frac{C_1\log t}{n_{m,i}(t)}})^2}{2var(\Tilde{\mu}^m_i)}\}} \notag \\
    & \leq (\exp{\{-\frac{(\sqrt{\frac{C_1\log t}{n_{m,i}(t)}})^2}{2|M_H|\frac{k_i\sigma^2}{n_{m,i}(t)}}\}})(1-3P_t) + 3P_t\notag \\
    & = (1-3P_t)\exp{\{-\frac{C_1\log t}{2|M_H|k_i\sigma^2}\}} + 3P_t \leq  \frac{4}{t^2}
\end{align}
where the first inequality holds by Chernoff bound, the second inequality is derived by plugging in the above upper bound on $var(\Tilde{\mu}^m_i(t))$, and the last inequality results from then choice of $C_1$ that satisfies $\frac{C_1}{2|M_H|k_i\sigma^2} \geq 1$. 

In a similar manner, we obtain
\begin{align}
    & P(\Tilde{\mu}^m_i(t) + \sqrt{\frac{C_1\log t}{n_{m,i}(t)}} < \mu_i, n_{m,i}(t-1) \geq l) \leq  \frac{4}{t^2}. 
\end{align}

Plugging the concentration-type inequalities in, we derive
\begin{align*}
    E[n_{m,i}(t)] & \leq l + \sum_{t= L+1}^T4P_t + \sum_{t= L+1}^T4P_t + 0 \\
    &  \leq l + \frac{4\pi^2}{3} = O(\log{T}).
\end{align*}

Then we arrive at 
\begin{align*}
    E[T_1] & \leq \sum_{m \in M_H}\sum_{k=1}^K\Delta_kE[n_{m,k}(t)] + |M_H|Kl^{T-1} \\
    & \leq O(\log{T}) + |M_H|Kl^{T-1}
\end{align*}

Once again, by the regret decomposition, we obtain that 
\begin{align*}
   E[R_T]  & \leq E[(c+1)\cdot L + T_1 + T_2] \\
   & \leq (c+1)\cdot L + O(1) + O(\log{T}) + |M_H|Kl^{T-1} \\
   & = O(\log{T})
\end{align*}




\end{proof}

\end{proof}

\end{document}

%% file: math_commands.tex
\usepackage{amsmath,amsfonts,bm}

\def\eqref#1{equation~\ref{#1}}

\def\1{\bm{1}}

\DeclareMathAlphabet{\mathsfit}{\encodingdefault}{\sfdefault}{m}{sl}
\SetMathAlphabet{\mathsfit}{bold}{\encodingdefault}{\sfdefault}{bx}{n}

\DeclareMathOperator*{\argmax}{arg\,max}